\documentclass[journal]{IEEEtran}
\IEEEoverridecommandlockouts

\usepackage{cite}
 \usepackage{amsmath,amssymb,amsfonts,amsthm}
\usepackage[pdftex]{graphicx}
\usepackage{textcomp}
\usepackage{xcolor}
\usepackage{booktabs}
\usepackage{multirow}
\usepackage{multicol}
\usepackage[caption=false]{subfig}
\usepackage[T1]{fontenc}
\usepackage[utf8]{inputenc}
\usepackage{babel}
\usepackage{mathtools}
\usepackage{algorithm}
\usepackage{tikz}
\usepackage{pgfplots}
\usepackage{glossaries-extra}
\usepackage{siunitx}
\usepackage{algpseudocodex}
\usepackage{svg}
\usepackage{etoolbox}
\usepackage{balance}
\usepackage{microtype}
\usepackage{mathtools}
\usepackage[hidelinks]{hyperref}
\usepackage[capitalise]{cleveref}
\usepackage{pgfplotstable}

\usepackage{xcolor, soul}
\sethlcolor{red}

\setabbreviationstyle[acronym]{long-short}

\svgsetup{inkscapelatex=false}

\pgfplotsset{compat=newest,
every axis/.style={
    grid=major,
    xlabel near ticks,
    ylabel near ticks,
    legend pos=south east,
    cycle list/Set2-8,
    cycle multiindex* list={
    mark list\nextlist
    Set2-8\nextlist
    linestyles\nextlist
    },
    scale only axis,
    every axis plot/.append style={mark size=2pt, ultra thick},
    enlarge y limits=0.025,
    enlarge x limits=0,
    label style={font=\large},
    tick label style={font=\large},
    title style={font=\large},
}
}
\graphicspath{{figures/}}

\usepgfplotslibrary{colormaps}
\usepgfplotslibrary{colorbrewer}
\usetikzlibrary{pgfplots.statistics}
\usepgfplotslibrary{external}

\newtheorem{remark}{Remark}

\makeatletter
\newcommand\fs@betterruled{%
  \def\@fs@cfont{\bfseries}\let\@fs@capt\floatc@ruled
  \def\@fs@pre{\vspace*{5pt}\hrule height.8pt depth0pt \kern2pt}%
  \def\@fs@post{\kern2pt\hrule\relax}%
  \def\@fs@mid{\kern2pt\hrule\kern2pt}%
  \let\@fs@iftopcapt\iftrue}
\floatstyle{betterruled}
\restylefloat{algorithm}
\makeatother

\usepackage[none]{hyphenat}

\usepackage{comment}

\newtheorem{defn}{Definition}
\newtheorem{thm}{Theorem}
\newtheorem{lem}{Lemma}

\providecommand{\TextBurak}[1]{#1}

\def\BibTeX{{\rm B\kern-.05em{\sc i\kern-.025em b}\kern-.08em
    T\kern-.1667em\lower.7ex\hbox{E}\kern-.125emX}}
    
\DeclareMathOperator*{\argmax}{arg\,max}

\renewcommand{\vec}[1]{\boldsymbol{#1}}
\newcommand{\norm}[1]{\left|\left|#1\right|\right|}
\newcommand{\real}{\mathbb{R}}
\newcommand{\card}[1]{\vert{#1}\vert}

\algnewcommand\algorithmicforeach{\textbf{for each}}
\algdef{S}[FOR]{ForEach}[1]{\algorithmicforeach\ #1\ \algorithmicdo}

\pgfplotsset{compat=1.17}

\renewcommand{\vec}[1]{\mathbf{#1}}

\newcommand{\bv}{\vec{b}}

\newcommand{\dv}{\vec{d}}

\newcommand{\gpv}{\vec{g_p}}

\newcommand{\mv}{\vec{m}}
\newcommand{\nv}{\vec{n}}

\newcommand{\pv}{\vec{p}}

\newcommand{\rv}{\vec{r}}

\newcommand{\wv}{\vec{w}}
\newcommand{\xv}{\vec{x}}
\newcommand{\yv}{\vec{y}}
\newcommand{\zv}{\vec{z}}

\newcommand{\Mm}{\vec{M}}

\newcommand{\Pm}{\vec{P}}
\newcommand{\Qm}{\vec{Q}}
\newcommand{\Rm}{\vec{R}}

\newcommand{\Dc}{{\cal D}}

\newcommand{\Nc}{{\cal N}}

\newcommand{\Pc}{{\cal P}}

\newcommand{\CC}{\mathbb{C}}

\newcommand{\NN}{\mathbb{N}}
\newcommand{\RR}{\mathbb{R}}

\newcommand{\LB}{\left(}
\newcommand{\RB}{\right)}
\newcommand{\LP}{\left\{}
\newcommand{\RP}{\right\}}
\newcommand{\LSB}{\left[}
\newcommand{\RSB}{\right]}

\renewcommand{\log}[1]{\mathop{\mathrm{log}}\LB #1\RB}

\newcommand{\EE}{{\mathbb{E}}}
\newcommand{\Expect}[2]{\EE_{#1}\LSB #2\RSB}

\newcommand{\Dcval}{\Dc_\mathrm{val}}

\newcommand{\removed}[1]{}
\newacronym{ACM}{ACM}{adaptive coding and modulation}
\newacronym{ADC}{ADC}{analog-to-digital conversion}
\newacronym{AGC}{AGC}{automatic gain control}
\newacronym{AWGN}{AWGN}{additive white Gaussian noise}
\newacronym{BER}{BER}{bit error rate}
\newacronym{BLER}{BLER}{block error rate}
\newacronym{BP}{BP}{backpropagation}
\newacronym{BPTT}{BPTT}{backpropagation through time}
\newacronym{CE}{CE}{cross-entropy}
\newacronym{CFO}{CFO}{carrier frequency offset}
\newacronym{CSI}{CSI}{channel state information}
\newacronym{DAC}{DAC}{digital-to-analog conversion}
\newacronym{DL}{DL}{deep learning}
\newacronym{DFT}{DFT}{discrete Fourier transform}
\newacronym{FFT}{FFT}{fast Fourier transform}
\newacronym{GAN}{GAN}{generative adversarial network}
\newacronym{GRU}{GRU}{gated recurrent unit}
\newacronym{iid}{i.i.d.\@}{independent and identically distributed}
\newacronym{IFFT}{IFFT}{inverse fast Fourier transform}
\newacronym{KL}{KL}{Kullback-Leibler}
\newacronym{LSTM}{LSTM}{long short-term memory}
\newacronym{MDP}{MDP}{Markov decision process}
\newacronym{ML}{ML}{machine learning}
\newacronym{MLP}{MLP}{multilayer perceptron}
\newacronym{MIMO}{MIMO}{multiple-input multiple-output}
\newacronym{MSE}{MSE}{mean squared error}
\newacronym{NN}{NN}{neural network}
\newacronym{DNN}{DNN}{deep neural network}
\newacronym{OFDM}{OFDM}{orthogonal frequency-division multiplexing}
\newacronym{pdf}{pdf}{probability density function}
\newacronym{pmf}{pmf}{probability mass function}
\newacronym{PSNR}{PSNR}{peak signal to noise ratio}
\newacronym{RBF}{RBF}{Rayleigh block-fading}
\newacronym{ReLU}{ReLU}{rectified linear unit}
\newacronym{RL}{RL}{reinforcement learning}
\newacronym{RNN}{RNN}{recurrent neural network}
\newacronym{SFO}{SFO}{sampling frequency offset}
\newacronym{SNR}{SNR}{signal-to-noise ratio}
\newacronym{SINR}{SINR}{signal-to-interference-plus-noise ratio}
\newacronym{SGD}{SGD}{stochastic gradient descent}
\newacronym{wrt}{w.r.t.\@}{with respect to}

\newacronym{OAC}{OAC}{over-the-air computation}
\newacronym{MAC}{MAC}{multiple access channel}
\newacronym{SIC}{SIC}{successive interference cancellation}
\newacronym{TDMA}{TDMA}{time division multiple access}
\newacronym{NOMA}{NOMA}{non-orthogonal multiple access}
\newacronym{CL}{CL}{curriculum learning}
\newacronym{JSCC}{JSCC}{joint source-channel coding}
\newacronym{DeepJSCC}{DeepJSCC}{deep joint source-channel coding}
\newacronym{MTL}{MTL}{multi-task learning}
\newacronym{MIL}{MIL}{multi-instance learning}
\newacronym{DML}{DML}{deep metric learning}
\newacronym{IoT}{IoT}{Internet of Things}
\newacronym{SSIM}{SSIM}{structural similarity index measure}
\newacronym{MS-SSIM}{MS-SSIM}{multi-scale \gls{SSIM}}
\newacronym{DDPM}{DDPM}{denoising diffusion probabilistic models}
\newacronym{MVL}{MVL}{multi-view learning}
\newacronym{CNN}{CNN}{convolutional neural network}
\newacronym{LPIPS}{LPIPS}{learned perceptual image patch similarity}
\newacronym{BPG}{BPG}{Better Portable Graphics}
\newacronym{IoE}{IoE}{Internet of Everything}
\newacronym{V2X}{V2X}{vehicle-to-everything}
\newacronym{AR/VR}{AR/VR}{augmented/virtual reality}
\newacronym{DSC}{DSC}{distributed source coding}
\newacronym{ANN}{ANN}{artificial neural networks}
\newacronym{BCR}{BCR}{bandwidth compression ratio}
\newacronym{CIS}{CIS}{central inference server}

\newacronym{MV-OAC}{MV-OAC}{majority voting with \gls{OAC}}
\newacronym{BA-OAC}{BA-OAC}{belief averaging with \gls{OAC}}
\newacronym{WBA-OAC}{WBA-OAC}{weighted belief averaging with \gls{OAC}}
\newacronym{MV-Orth}{MV-Orth}{orthogonal majority voting}
\newacronym{BA-Orth}{BA}{orthogonal belief averaging}
\newacronym{WBA-Orth}{WBA-Orth}{orthogonal weighted belief averaging}
\newacronym{CD}{CD}{critical distance}
\newacronym{DP}{DP}{differential privacy}

\newacronym{RR}{RR}{randomized response}
\newacronym{FL}{FL}{federated learning}

\begin{document}

\title{Private Collaborative Edge Inference\\ via Over-the-Air Computation
\thanks{
This work extends from our preliminary study presented at the 2022 IEEE International Symposium on Information Theory (ISIT)~\cite{yilmaz2022over}.

S.\ F.\ Yilmaz and D.\ Gündüz are with Department of Electrical and Electronic Engineering, Imperial College London, United Kingdom. Email: \{s.yilmaz21, d.gunduz\}@imperial.ac.uk.

B.\ Hasırcıoğlu is with The Alan Turing Institute, United Kingdom. This work was carried out when he was a Ph.D.\ student at the Information Processing and Communications Lab (IPC Lab) in the Department of Electrical and Electronic Engineering, Imperial College London, United Kingdom. Email: burakhasircioglu@gmail.com.

L.\ Qiao is with the School of Information
and Electronics, Beijing Institute of Technology, Beijing 100081, China. This work was carried out when he was a visiting student at the Information Processing and Communications Lab (IPC Lab) in the Department of Electrical and Electronic Engineering, Imperial College London, United Kingdom. Email: qiaoli@bit.edu.cn.

The present work has received funding from the European Union’s Horizon 2020 Marie Skłodowska Curie Innovative Training Network Greenedge (GA. No. 953775). Horizon Europe SNS project "6G-GOALS" (grant 101139232), and from the UK Research and Innovation (UKRI) for the projects Al-R: Artificial Intelligence in the Air (ERC Consolidator Grant, EP/X030806/1) and SONATA: Sustainable Computing and Communication at the Edge (EPSRC-EP/W035960/1). For the purpose of open access, the authors have applied a Creative Commons Attribution (CC BY) license to any Author Accepted Manuscript version arising from this submission.}
}

\author{\IEEEauthorblockN{Selim F. Yilmaz, Burak Hasırcıoğlu, Li Qiao and Deniz Gündüz}}

\maketitle

\begin{abstract}
We consider collaborative inference at the wireless edge, where each client's model is trained independently on its local dataset. Clients are queried in parallel to make an accurate decision collaboratively. In addition to maximizing the inference accuracy, we also want to ensure the privacy of local models. To this end, we leverage the superposition property of the multiple access channel to implement bandwidth-efficient multi-user inference methods. We propose different methods for ensemble and multi-view classification that exploit over-the-air computation (OAC). We show that these schemes perform better than their orthogonal counterparts with statistically significant differences while using fewer resources and providing privacy guarantees. We also provide experimental results verifying the benefits of the proposed OAC approach to multi-user inference, and perform an ablation study to demonstrate the effectiveness of our design choices. We share the source code of the framework publicly on Github to facilitate further research and reproducibility.
\end{abstract}


\section{Introduction}
The increasing adoption of \gls{IoT} devices results in the collection and processing of massive amounts of mobile data at the wireless edge. Conventional centralized machine learning (ML) methods are impractical for edge applications due to privacy concerns and limited communication resources. Implementing decentralized ML models at the edge solves this issue, and thus, {\it edge learning} and {\it edge inference} have attracted significant attention over the recent years~\cite{chen2021distributed,wu2019machine,zhou2019edge,lan2021progressive,gunduz2020communicate}. Edge learning aims to train large ML models in a distributed setting, whereas edge inference aims to make inferences in a distributed manner at the edge. 

Although collaborative training (e.g., \gls{FL}) at the edge can bring significant advantages, it requires coordination and communication across nodes. Moreover, limited wireless resources are a major bottleneck, and noise, interference, and lack of accurate channel state information can prevent or slow down the convergence of learning algorithms, or result in reduced accuracy~\cite{chen2021guest}. Therefore, in this paper, we consider collaborative inference using independently trained models at the edge nodes. While a growing body of work studies distributed training over wireless networks, the literature on distributed wireless inference, particularly using deep learning techniques, is relatively limited~\cite{jankowski2020wireless,jankowski2020joint,shao2021learning}. Moreover, beyond a few exceptions, such physical factors affecting edge inference have not been explored in the literature~\cite{chen2021distributed}.

We treat the resultant problem as a collaborative edge inference problem, where the individual hypotheses of the clients need to be conveyed to the \gls{CIS}, and combined for the most accurate decision. Moving sensing to the edge reduces latency and improves privacy since it removes the requirement to communicate the data or the models. Often the data is not accumulated at a single client or cannot be shared between clients due to privacy or latency constraints. When data offloading or collaborative training are not possible, locally trained models are suboptimal and a collaborative inference solution is needed that incorporates the decisions at different clients. We, therefore, introduce a distributed inference solution that incorporates all the data and models while limiting privacy leakage during inference time through \gls{DP} guarantees.

Privacy is an important concern in all ML applications since the data used for training can reveal sensitive information about its owner. In the case of ensemble inference, when the models are queried, their outputs may reveal sensitive information about their training sets. For instance, even when an adversary has black-box access to a model, whether or not a data sample is used during training can be inferred via membership inference attacks~\cite{shokri2017membership}, or even the whole model can be reconstructed via model inversion attacks~\cite{tramer2016stealing}. Hence, even if adversaries can only observe the inference results, we need to introduce some additional mechanisms to protect the sensitive information. \Cref{fig:general_scheme} summarizes the architecture for collaborative inference with \gls{DP} guarantees.

\Gls{DP} guarantees can be obtained by introducing additional randomness to the output, such as adding noise, at the expense of some accuracy loss. Since \gls{DP} bounds the amount of information leaked about the individuals, \gls{DP} mechanisms make black-box attacks less effective. One approach to provide \gls{DP} guarantees to ML is differentially private training~\cite{abadi2016deep}. Typically, Gaussian noise is added to the gradients during training, where the noise variance is determined according to the desired privacy level. This approach is extended to a federated setting in~\cite{geyer2017differentially,malekzadeh2021dopamine}.
However, \gls{DP} training provides only a fixed \gls{DP} guarantee, and we cannot operate at different privacy-utility trade-offs during inference, which may be beneficial when serving users with different levels of trustworthiness. Moreover, \gls{DP} training does not prevent the model stealing attacks since the model can be still reconstructed via black-box access to it. Hence, in this paper, we focus on embedding privacy-preserving mechanisms into the inference phase. We simply lift \gls{DP} training assumption on the models and provide a systematic method for controllable privacy-utility trade-off for distributed inference.

\begin{figure}[t!]
    \centering
    \includegraphics[width=\columnwidth]{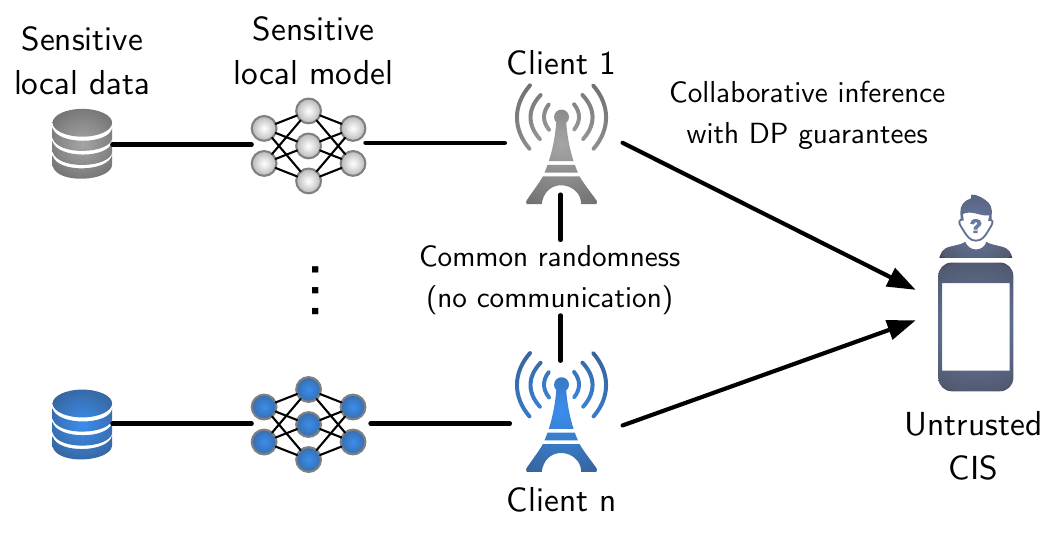}
    \caption{Overview of collaborative inference with \gls{DP} guarantees.}
    \label{fig:general_scheme}
\end{figure}

In a recent line of work~\cite{amiri2020machine,amiri2020federated,zhu2019broadband,yilmaz2023distributed}, it has been shown that, in distributed training tasks, superposition property of the \gls{MAC} can be exploited to use communication resources much more efficiently and to significantly improve the learning and transmission performance. Instead of conventional digital communication, in \gls{OAC}, clients transmit their updates simultaneously in an uncoded manner such that the receiver automatically gets the aggregated signal via the \gls{MAC} (please refer to~\cite{csahin2023survey} for a comprehensive overview of \gls{OAC} techniques and applications).
Hence, besides communication efficiency, \gls{OAC} also helps preserve the privacy of the clients since individual signals transmitted by the clients are not observed by the receiver~\cite{hasirciouglu2021private}. Moreover, the noise and channel fading acting on the aggregated signal at the receiver can be effective at preserving the privacy of all the signals transmitted by the clients~\cite{seif2020wireless, liu2020privacy, sonee2020efficient, seif2021privacy, hasirciouglu2021private}.
In this work, we extend the use of \gls{OAC} beyond distributed training and exploit it for efficient and private distributed edge inference. 


\subsection{Contributions}
The main contributions of our work are summarized as follows:
\begin{enumerate}
    \item We introduce a novel distributed inference framework at the wireless edge that exploits \gls{OAC} while providing bandwidth efficiency and variability. To the best of our knowledge, the conference version of this paper \cite{yilmaz2022over} was the first in the literature to exploit OAC for distributed inference.
    \item We provide flexible privacy guarantees depending on the scenario without imposing any restrictions on the training phase.
    \item We systematically compare and discuss the privacy of the introduced collaborative classification methods, and show that the proposed framework with \gls{OAC} performs significantly better than its orthogonal and digital counterparts while using fewer wireless resources under privacy constraints. We also provide statistical significance tests to show the reliability of the obtained results.
    \item To facilitate further research and reproducibility, we publicly share the source code of our framework on \href{https://github.com/ipc-lab/collaborative-inference-oac}{github.com/ipc-lab/collaborative-inference-oac}.
\end{enumerate}

\subsection{Organization of This Article}
The rest of the article is organized as follows. In \cref{sec:related_works}, we describe related works from wireless communications and machine learning literature. In \cref{sec:system_model}, we define distributed inference problem over a \gls{MAC} and describe our threat model and target problem. In \cref{sec:methodology}, we introduce our novel distributed inference method for the wireless edge that exploits the superposition property of \gls{MAC}. In \cref{sec:numerical_results}, we demonstrate the performance and bandwidth gains of our introduced methodology. We conclude our work in \cref{sec:conclusion}.
\section{Related Work}
\label{sec:related_works}



\subsection{Distributed Inference at the Edge}

For \gls{IoT} applications, moving inference from cloud to edge can be desirable or needed in both static and dynamic environments \cite{ chen2021distributed}. Static environments, e.g., surveillance cameras or sensor networks, might desire edge inference due to latency and privacy constraints~\cite{shlezinger2022collaborative, hu2019deephome}. Dynamic environments are even more challenging with the same demands due to their high mobility. In both static and dynamic environments, edge devices can even generate data faster than their upload speed to the cloud, which makes it impossible to employ a cloud-centric solution while utilizing all the data~\cite{chen2021distributed}.

Edge inference in dynamic environments, such as the ones including self-driving cars, drones and mobile phones, might require adaptive solutions due to latency, privacy, and accuracy constraints \cite{zhou2019edge,chen2021distributed,Cui2023RecUP-FL:,Yang2023Energy-Efficient}. In a collaborative system, client participation depends on the privacy, security, energy and utility requirements~\cite{malik2023collaborative,wang2018privacy}. Clients may choose not to participate due to various reasons such as (1) hardware failures, (2) to reduce or balance the power consumption, (3) to prevent accuracy drop of the \gls{CIS} due to noisy data receipt or unconfident local prediction, (4) to prevent leakage to possible eavesdroppers, (5) to prevent transmission over a low-quality wireless channel, and (6) to improve \gls{DP} guarantees. In this study, we focus on amplifying \gls{DP} guarantees by random participation.


\subsection{OAC for FL}

\Gls{FL}~\cite{konevcny2016federated} has found applications in many areas including, but not limited to, finance~\cite{long2020federated}, medicine~\cite{nguyen2021federated, shiri2023decentralized} and wireless communications~\cite{FrancescWilhelmi2022,mashhadi2021federated,mashhadi2021fedrec}. In \gls{FL}, multiple devices are coordinated by a central server to train a shared neural network model. \Gls{FL} lifts the requirement to communicate data samples, i.e., the data stays at local devices, and allows collaborative training to benefit from distributed local data samples. \gls{OAC} has been shown to have a considerable amplification effect on \gls{DP} thanks to the natural aggregation of multiple model updates and the channel noise over the wireless medium. Such an amplification effect of \gls{OAC} has been studied previously in several works providing better convergence, privacy and accuracy for \gls{FL}~\cite{seif2020wireless, liu2020privacy, sonee2020efficient, seif2021privacy, hasirciouglu2021private}. However, this framework assumes model training among physically colocated nodes, which may limit its practical applicability since such collaborative training may require an enormous amount of communication and computation resources. Instead, in this paper, we focus on edge inference with relatively simpler models that are trained independently.

\subsection{OAC for Edge Inference}
Following our initial work~\cite{yilmaz2022over}, \gls{OAC} has been adopted to multi-device edge inference in several studies~\cite{Kaibin2023Task,zhuang2023integrated,WU2023features,LIU2023over}. Specifically, in~\cite{Kaibin2023Task,zhuang2023integrated,WU2023features}, feature vectors are extracted at multiple local devices and transmitted via \gls{OAC}, then the neural network employed at the server performs inference based on the superposed feature vectors. Due to observations from multiple devices, the server can achieve better inference results. Moreover, power allocation, transmit precoding, and receive beamforming are optimized to maximize the inference-oriented criterion in~\cite{Kaibin2023Task} and~\cite{zhuang2023integrated}. In addition, contrastive learning has been used to exploit feature correlations among devices to maximize the image retrieval accuracy in~\cite{WU2023features}. The authors of~\cite{LIU2023over} proposed an \gls{OAC}-based multi-view pooling technique to improve classification accuracy with less communication latency. However, to the best of our knowledge, only~\cite{yilmaz2022over} and the current paper consider the privacy aspects of distributed edge inference.




\subsection{Ensemble and Multi-View Classification}
Ensemble learning methods combine multiple hypotheses instead of constructing a single best hypothesis to model the data~\cite{dietterich2002ensemble}. In ensemble learning, each hypothesis votes for the final decision, where votes can have weights depending on their confidence. It is generally intractable to find the optimal hypothesis, and choosing a model among a set of equally good models has the risk of choosing the model that has worse generalization performance; however, averaging these models would reduce this risk~\cite{dietterich2002ensemble,bishop1995neural}. Furthermore, weighted or voting-based ensemble methods have theoretical guarantees; for example, it can be shown that the expected error of an averaging ensemble of models is not greater than the average of the expected errors of individual models with a mean square objective~\cite{bishop1995neural}.

Ensemble methods achieve higher accuracy when local models are diverse~\cite{sagi2018ensemble}. Our setting is similar to \textit{dagging}~\cite{ting1997stacking}, where individual models of the ensemble are trained on disjoint datasets. Dagging is a very similar technique to the well-known \textit{bagging}~\cite{breiman1996bagging}, with the difference in disjoint training subset sampling like ours instead of sampling with replacement. We employ models trained on disjoint datasets due to the nature of our target problem, which improves the diversity and privacy of the local models.

Multi-view classification deals with multiple feature perspectives, e.g., different views of the same scene, while ensemble classification performs classification on the same whole instance. Both approaches aim to improve the classification performance. In multi-view classification, the inputs are often correlated, and therefore the aim is to utilize these correlations among different views~\cite{shi2022multi}. To fuse information from other modalities or views, data-level, feature-level or decision-level fusion techniques have been commonly employed~\cite{mandira2019spatiotemporal,giritliouglu2021multimodal}. Among them, decision-level fusion does not require any data or feature sharing between models, and therefore improves latency and privacy. For this purpose, decision-level fusion is a commonly employed technique while benefiting from the local models or datasets of individual clients.

Averaging, weighted averaging and voting are common methods for decision-level fusion~\cite{dietterich2002ensemble}. Weighted averaging is more general than simple averaging and voting as these methods are special cases of weighted averaging. Various empirical studies show that simple averaging is not worse than weighted averaging~\cite{xu1992methods,ho1994decision}. Yet, in general, weighted averaging is a better choice when individual models have significantly different performances on the target task~\cite{Zhou2021}. Simple averaging can be a better choice when individual models have similar performances since weights need to be determined, and they may not be reliable due to noise and insufficient amount of data~\cite{Zhou2021}.

\section{System Model and Problem Definition}
\label{sec:system_model}
\subsection{Notation}
Unless stated otherwise; boldface lowercase letters denote vectors (e.g., $\pv$), boldface uppercase letters denote matrices (e.g., $\Pm$), non-boldface letters denote scalars (e.g., $p$ or $P$), and uppercase calligraphic letters denote sets (e.g., $\Pc$). Blackboard bold letters denote function domains (e.g., $\mathbb{P}$). $\RR$, $\NN$, $\CC$ denote the set of real, natural and complex numbers, respectively. $\card{\Pc}$ denotes the cardinality of set $\Pc$. We define $[n]\triangleq\{1,2,\cdots,n\}$, where $n\in\NN^+$.

\subsection{System Model}
We consider privacy-preserving multi-user classification at the wireless edge for both ensemble and multi-view cases. In our model, we consider $n$ clients each with a separate model $f_i:\RR^l\to \RR^k,i\in[n]$, trained on dataset $\Dc_i \subset \Dc$ for either a multi-class or a binary classification task, where $\Dc = \bigcup_{i \in \LSB n \RSB} \Dc_i$, and $\Dc_i\cap \Dc_j=\emptyset$ for clients $i \neq j$. Each client also has access to a common validation dataset $\Dcval$.

A new query to be classified arrives at each timestep $t$, and all clients receive a view of the query, $\xv_{i,t} \in \RR^l$.
Then, each participating client makes an inference with its local model denoted by $f_i(\xv_{i,t}) \in \RR^k$ and then encodes it into $d$ channel symbols, which is the total available bandwidth per query.

\begin{remark}
The target task becomes ensemble classification when all devices receive the same view $\xv_t$ as query, i.e., $\xv_{1,t} = \cdots = \xv_{n,t} = \xv_t$.
\end{remark}

The clients are connected to a \gls{CIS} via a wireless medium, and, at time $t$, we assume each client $i$ knows its channel gain to the \gls{CIS}, $h_{i,t}\in\RR$. In our setting, we assume that $h_{i,t}$'s are normally distributed with a zero mean and a variance $\sigma^2_h$. We further assume that the channel gains change across clients and time steps, but they stay the same per inference round; that is, the channel gain of a client is the same for all $d$ channel uses.
To reduce the total power consumption and to amplify the privacy guarantees, we consider random participation of the clients in each inference round such that each client $i$ independently participates with probability $p$. To limit the power consumption, only the clients whose channel gains are larger than a certain threshold participate in the inference process. This is one of the sources of randomness determining $p$. Hence, $p$ is a tunable parameter via such a transmission threshold. If necessary, via additional randomness, $p$ can be made even smaller.

Let $\vec{y}_{i,t}\in\RR^d$ denote the signal transmitted by client $i$. The received signal at \gls{CIS} is 
\begin{equation}
\vec{z}_t = \sum_{i \in \mathcal{P}_t}  h_{i,t} \vec{y}_{i,t} +  \nv_t,
\label{eq:channel_model}
\end{equation}
where $\nv_t \in \RR^d$ is the \gls{iid} \gls{AWGN} with zero mean and variance $\sigma^{2}_n$, i.e., $\nv_t \sim \mathcal{N}(\vec{0}, \sigma^2_n \vec{I}_k)$. We enforce an average power constraint $P$, which is denoted by the following:
\begin{align}
\Expect{}{\norm{\yv_{i,t}}_2^2} \leq P.
\label{eq:system_power_constraint}
\end{align}

After receiving $\vec{z}_t \in \RR^d$, \gls{CIS} processes it via function $s:\RR^d\to[k]$ and outputs the most probable class. The system model is illustrated in \cref{fig:main_figure}.

\subsection{Threat Model and Target Problem}
In our problem, the purpose is to limit the privacy leakage of clients' local models, which is equivalent to limiting the leakage about the individual datasets $\mathcal{D}_i, \forall i \in [n]$ since local models are trained on disjoint datasets.
In our threat model, we assume all the clients follow the protocol, and the \gls{CIS} is honest but curious, i.e., it does not deviate from the protocol, but by exploiting the signals received from the clients, it may try to infer sensitive information about the clients' models and datasets.  Hence, our goal is to limit the leakage to \gls{CIS} about the local models from $\vec{z}_t$ while trying to maximize the inference accuracy.


\begin{remark}
    Although our channel model is defined in real numbers $\RR$ for simplicity, it can be extended to complex channels via mapping half of the vectors or function outputs to the real part and the rest to the imaginary part of a complex number.
\end{remark}
\section{Methodology}\label{sec:methodology}
In this section, we gradually introduce the modules of our framework.

\begin{figure}[t!]
    \centering
    \includegraphics[width=\columnwidth]{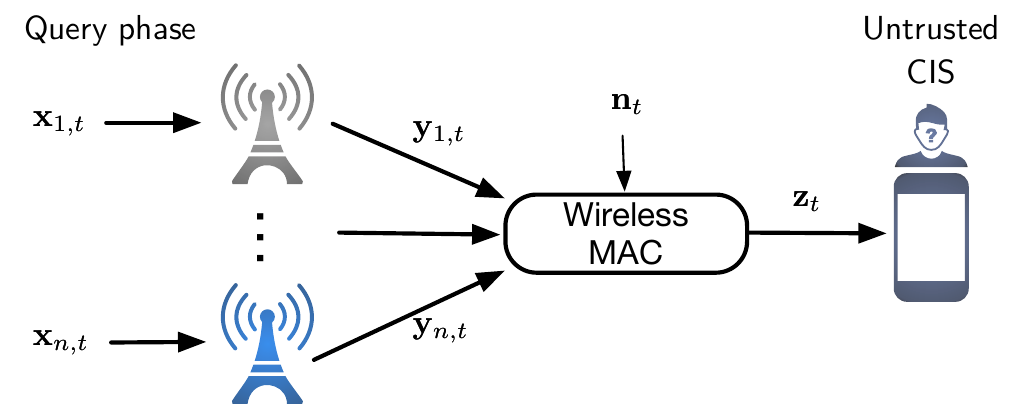}
    \caption{Overview of our framework for collaborative private inference at the wireless edge. 
    }
    \label{fig:main_figure}
\end{figure}

\subsection{Fusion Methods}
In the following, we present alternative methods to perform local prediction for participating clients, having received the query $\xv_{i,t}$. 

Let $\rv_{i,t} \in \real^k$ be a vector containing classifier scores (beliefs) for each class at client $i$ for the sample received at time slot $t$, where $k$ is the number of classes, and the $j^\mathrm{th}$ element of $\rv_{i,t}$, denoted by $(\rv_{i,t})_j$, contains the score for class $j$. We normalize the sum of the scores in $\rv_{i,t}$ to 1, i.e., $\norm{\rv_{i,t}}_1 = 1$, and hence, the maximum possible score of a class is 1.

We now present three different models to be used as $f_i$: {\it \gls{BA-OAC}}, {\it \gls{WBA-OAC}}, {\it \gls{MV-OAC}}.

{\it \Gls{BA-OAC}} method averages beliefs of the participating clients for all the classes, and the \gls{CIS} selects the class with the highest total score. Thus, it uses the following model for client $i$:
\begin{equation}
    \tilde{f}_i (\xv_{i,t}) = \rv_{i,t}.
\end{equation}

{\it \Gls{WBA-OAC}} method employs normalized weighted average of the beliefs of the participating clients by class-wise accuracy. Then, similarly to {\it \gls{BA-OAC}}, the \gls{CIS} selects the class with the highest total score. Therefore, it uses the following function at every client $i$:
\begin{equation}
    \tilde{f}_i (\xv_{i,t}) =  \wv_i \odot \rv_{i,t},
\end{equation}
where $\wv_i \in \RR^k$ is a vector containing normalized per-class accuracy on $\Dcval$ for client $i$ satisfying $\norm{\wv_i}_1=1$, and $\odot$ denotes the element-wise multiplication. If real-time validation data or test data are available and do not violate privacy constraints, these can be used to update the weights accordingly instead of $\Dcval$.

\begin{defn}
$\mathrm{OneHot}(j, l)$ function outputs an $l$-dimensional one-hot vector for $j \leq l$, where only the $j^\mathrm{th}$ dimension is $1$ and the rest are $0$.
\end{defn}

{\it \Gls{MV-OAC}} method allows participating clients to vote for a class and the \gls{CIS} selects the class with the highest number of votes. Hence, at client $i$, we have
\begin{equation}
    \tilde{f}_i (\xv_{i,t}) = \mathrm{OneHot} \left( \argmax_{j\in[k]} (\rv_{i,t})_j, k \right).
\end{equation}

While {\it \gls{BA-OAC}} and {\it \gls{WBA-OAC}} combine local discriminative scores,  {\it \gls{MV-OAC}} combines predicted labels. We then apply mean centering by subtracting $\tilde{f}_i ( \xv_{i,t} )$'s mean, which is $ \LB\hat{f}_i ( \xv_{i,t} )\RB_j = \frac{1}{k}$, $\forall j \in [k]$, for {\it \gls{MV-OAC}}, i.e.,
\begin{equation}
    f_i \LB \xv_{i,t} \RB = \tilde{f}_i ( \xv_{i,t} ) - \hat{f}_i ( \xv_{i,t} ).
    \label{eq:power_shift}
\end{equation}
Therefore, when transmitted over the channel, $f_i \LB \xv_{i,t} \RB$ uses less transmission power for {\it \gls{MV-OAC}}, i.e., $\norm{f_i \LB \xv_{i,t} \RB}_2^2 < \norm{\tilde{f}_i ( \xv_{i,t} )}_2^2$.




\subsection{Ensuring Privacy}\label{subsec:ensuring-privacy}
Next, we explain how we make our inference procedure privacy-preserving by introducing some randomness. First, we formally define DP for our ensemble inference task.

Let $\mathcal{L}, \mathcal{L}' \in \mathbb{L}$ be the sets of local models of the clients, which differ at most in one of the clients, i.e., $\mathcal{L}=\{f_j\}\cup \{f_i : i\in[n]\setminus j\}$ and $\mathcal{L}'=\{f_j'\}\cup \{f_i : i\in[n]\setminus j\}$ such that $f_j\neq f_j'$ and $\mathbb{L}$ is the set of all possible local models of $n$ clients. Such $\mathcal{L}$ and $\mathcal{L}'$ are called \emph{neighbouring} sets. In our case, since we aim to protect the local models' privacy against CIS, we require that the CIS cannot distinguish between $\mathcal{L}$ and $\mathcal{L}'$ by observing $\zv_t$. Hence, the DP guarantees considered in the paper are all local-model-level DP, i.e., the replacement of one client's local model with another model does not generate a distinguishable effect on the observation of the CIS.


Local-model-level privacy guarantees might seem too conservative, but 
besides protecting against inferring sensitive features of the local datasets, local-model-level privacy provides protection against model stealing attacks, i.e., the \gls{CIS} will not be able to accurately reconstruct the local model parameters.


Next, we formally define the notion of DP which characterizes the indistinguishability of $\zv_t$ against the local model sets $\mathcal{L}$ and $\mathcal{L}'$. For this, for a clearer notation, let us consider the signal observed by the CIS, $\zv_t$, as the output of a randomized function $M(\boldsymbol{X}_t,\mathcal{L})$, which represents all the local prediction, randomization and transmission phases, where $\boldsymbol{X}_t\in \real^{l\times n}$ is the matrix whose $i^{th}$ column represents the query received by client $i$.

\begin{defn}
Consider the signal observed by the CIS, $\zv_t$, as the output of a randomized function $M:\real^{l\times n}\times\mathbb{L}\to \real^d$  and let $\mathcal{L}$ and $\mathcal{L}'$ be two possible neighbouring model sets. For $\varepsilon>0$ and $\delta \in [0,1)$, $M$ is called $(\varepsilon,\delta)$-DP if
\begin{equation}
    \Pr(M(\boldsymbol{X}_t,\mathcal{L})\in \mathcal{R}) \leq e^\varepsilon \Pr(M(\boldsymbol{X}_t, \mathcal{L}')\in \mathcal{R}) + \delta,
\end{equation}
for all neighboring pairs $(\mathcal{L},\mathcal{L}')$, $\forall \boldsymbol{X}_t \in \real^{l\times n}$  and $\forall$ $\mathcal{R}\subset \mathbb{R}^d$.
\end{defn}


In the above definition, $\varepsilon$ indicates the amount of privacy loss and hence, smaller $\varepsilon$ implies a stronger privacy guarantee. On the other hand, $\delta$ characterizes the failure probability of this guarantee, i.e., bad events in which $\Pr(M(\mathcal{L})\in \mathcal{R}) \leq e^\varepsilon \Pr(M(\mathcal{L}')\in \mathcal{R})$ does not hold. Hence, $\delta$ should be close to zero for a strong privacy guarantee. Choosing appropriate values for $\varepsilon$ and $\delta$ is crucial for balancing between privacy and utility, and is still an active area of research. In principle, to have a strong privacy guarantee, choosing $\varepsilon\leq 1$ is preferable. It means that when a private sample is included, which is the private model of a device in our use case, the probability of any output value does not exceed more than $\sim 2.72$ times the probability of the same output value when the sample is not used. However, in practice, depending on the use case, considering privacy vs. utility tradeoff, much larger $\varepsilon$ values would be enough. For example, in \cite{carlini2022membership}, it has been demonstrated that using even a relatively large value of $\varepsilon=8$ can be sufficient to deteriorate the performance of membership inference attacks. Moreover, in \cite{hayes2024bounding}, it has been shown that using $\varepsilon < 100$ would be enough to prevent data reconstruction attacks in MNIST and CIFAR-10 datasets. For further discussions on how $\varepsilon$ and $\delta$ should be chosen in practice, we refer the readers to~\cite{dwork2014algorithmic, steinke2022composition, carlini2022membership, hayes2024bounding, lowy2024does, lee2011much}.

To achieve DP guarantees, the output released to an adversary should be randomized. In our paper, we consider releasing a noisy version of model outputs, $f_i(\xv_{i,t})$, for each client by adding Gaussian noise~\cite{dwork2006calibrating}. Each client can generate a noisy version of its model prediction as follows:
\begin{equation}
    g_i (\xv_{i,t}) = f_i(\xv_{i,t}) + \mv_{i,t},
    \label{eq:gi}
\end{equation}
where $\mv_{i,t} \sim \mathcal{N}(\vec{0}, \sigma^{2}_\mathrm{client} \vec{I}_k)$. One of the main advantages of OAC is that the noise terms added by different clients are automatically aggregated at CIS. Thus, it has a further privacy amplification effect. We provide the analysis of the privacy guarantees achieved by our framework in \cref{sec:privacy-analysis}. This analysis reveals that DP guarantees are directly dependent on the variance of the aggregated noise at the CIS. Hence, to obtain DP guarantees, each client should add a Gaussian noise with $\sigma_{\mathrm{client}}^2=\sigma^2/|\mathcal{P}_t|$, where $\sigma^2$ is a constant depending on the desired DP guarantees \TextBurak{and $\mathcal{P}_t$ is the set of participating clients}. Hence, we assume that the number of participating clients is known by the other participating clients, but is secret from the CIS. 

\begin{remark}
Note that in OAC, $\zv_t$ already has channel noise that provides some degree of privacy guarantees, which can even be amplified by fading~\cite{liu2020privacy}. However, to achieve the desired level of DP, channel noise power may not be sufficient. First of all, channel noise may not follow a Gaussian distribution in practice (which is often used as worst-case assumption for channel capacity). More importantly, the clients cannot control or reliably know the noise variance or channel gain, as they rely on the CIS for this information. Thus, it is not a reliable source of randomness~\cite{hasirciouglu2021private}, and we ignore the channel noise and fading while analysing privacy guarantees. Instead, we have each client add some additional Gaussian noise before releasing their contributions. Hence, in reality, the privacy guarantees are slightly better than the ones we obtain in this work due to the presence of additional channel noise. If we were not ignoring the channel noise, each client would need to add less noise to the decision vectors, i.e., noise variance of $(\sigma^2 - \sigma_n^2)/|\mathcal{P}_t|$ would suffice instead of $\sigma^2/|\mathcal{P}_t|$. 
\end{remark}

We discuss the \gls{DP} guarantees of our method in \cref{sec:privacy-analysis}. In the next section, we describe the transmission method of the clients.

\subsection{Transmission}
\label{sec:methodology_transmission}
Before transmission, all the clients check whether they participate or not at time $t$. That is, each client participates with probability $p$ independent of other clients. If no client participates at time $t$, which has a very low probability, then the clients repeat the coin tosses for their participation until at least one client participates, i.e., until $\card{\Pc_t} > 0$. Therefore, each client needs to know the number of participating clients at time $t$, i.e., $\card{\Pc_t}$. Clients can decide on the number of participating clients and their identities via a procedure that is unknown to the \gls{CIS}. Such a procedure can be executed without communication by leveraging either the randomness in channel gains or common randomness, with client selection determined by a pseudo-random variable derived from a shared seed.

Furthermore, we use linear projection before transmission to adapt to the available bandwidth $d$. When the number of classes $k$ is smaller than $d$, it is reasonable to employ the whole bandwidth to improve the transmission quality against channel noise. However, the number of classes $k$ can be much larger than $d$, e.g., the number of labels ranges from thousands to billions in extreme classification problems~\cite{choromanska2013extreme,liu2017deep}. In this case, to enable latency critical applications~\cite{schulz2017latency}, it is better to project the decision vector $g_i(\xv_{i,t})\in \RR^{k}$ to a $d$-dimensional space to reduce the bandwidth, albeit with lower reliability. Therefore, for more generality, we adopt a linear projection of the decision vector $g_i(\xv_{i,t})\in \RR^{k}$ to dimension $d$.

Our objective is to recover the superposition of projected signals at the \gls{CIS} with the least error possible, while we can expand or collapse the bandwidth. In particular, we employ the same projection matrix $\Pm \in \RR^{d \times k}$ at all the clients to produce $d$-dimensional transmit vectors. To generate $\Pm$, we first sample $\Mm \sim \Nc \LB \mathbf{0}_c, \mathbf{I}_c \RB$, where $c=\LP d, k \RP$. Then, we apply QR factorization to $\Mm$ using the procedure in~\cite{mezzadri2006generate} to generate an orthogonal matrix $\Qm$ as:
\begin{align*}
    \Qm^{'}, \Rm &= \mathrm{QRFactorization} \LB \Mm \RB,\\
    \dv &= \mathrm{Sign} \LB \mathrm{Diagonal} \LB \Rm \RB \RB,\\
    \Qm &= \dv \cdot \Qm^{'},
\end{align*}
where $\mathrm{Diagonal}$ extracts the diagonal elements of the input matrix as a vector, and $\mathrm{Sign}$ replaces negative elements with $-1$ and positive elements with $+1$.
Then, we select the first $d$ rows and $k$ columns as $\Pm = \Qm_{1:d,1:k}$. Since the \gls{CIS} is only interested in the final decision, it only needs to pinpoint the maximum value of the superposed decision vector rather than estimating the entire decision vector. In addition, most of the clients may make the same decision; hence, it is feasible to identify the maximum value even when $d < k$. 

\begin{remark}
    Any orthogonal matrix can be used instead of $\Qm$. When $d < k$, $\Pm$ can also be designed as a matrix. e.g., a Gaussian or Bernoulli random matrix, that satisfies the restricted isometry property of compressed sensing~\cite{Donoho}.
\end{remark}

We need to make sure that each client's noisy decision vector $g_i(\xv_{i,t})$ is received at the \gls{CIS} at the same power level. Recall that the channel gain for each client is assumed to be perfectly known by that client, which then employs channel inversion to cancel its effect. Thus, each client scales its signal by $1/h_{i,t}$ prior to transmission. Note that, since a client does not participate in the inference process if its channel has a low gain, this scaling does not result in excessive power usage. The \gls{CIS} may require a specific power level for the reception of the signals depending on the available power of the clients or the channel's power constraint. Typically, an average power constraint is imposed. We satisfy the average power constraint in \cref{eq:system_power_constraint} by scaling via $\gamma$, whose calculation is provided in \cref{subsec:power_normalization}. Lastly, we normalize by $\card{\Pc_t}$ so that the receiver receives the average of all the users.

After these steps, the signal transmitted by the $i^\mathrm{th}$ client is given by 
\begin{equation}
    \vec{y}_{i,t} =
    \begin{dcases*} 
       \gamma \Pm \frac{ g_i (\xv_{i,t})}{\card{\Pc_t} h_{i,t}}, & if $i\in \mathcal{P}_t$,   \\ 
      \vec{0}, & otherwise,
    \end{dcases*}
    \label{eq:yit}
\end{equation}
where $\Pc_t$ is the set of participating clients.

Now, all participating clients transmit their signals synchronously over a \gls{MAC}. Thanks to the common randomness, it is reasonable to assume that multiple clients are symbol-level synchronized as in~\cite{amiri2020federated} and~\cite{zhu2019broadband}. Hence, their signals are superposed at the \gls{CIS}. Nevertheless, this assumption may not hold in large-scale networks, where achieving synchronization is challenging due to hardware variations, channel conditions, and geographical distances. In asynchronous settings, synchronization errors can lead to significant interference, reducing overall system accuracy. For approaches to mitigate these issues, we refer the reader to~\cite{shao2021federated,csahin2023survey}.


\subsection{Final Decision by the CIS}

The signal received by the \gls{CIS} at time $t$ after passing through the channel defined in \cref{eq:channel_model} is given as follows:
\begin{align}
\zv_t &=  \sum_{i \in \mathcal{P}_t}  h_{i,t} \vec{y}_{i,t} +  \nv_t \nonumber \\
&\stackrel{(a)}{=} \sum_{i \in \mathcal{P}_t}  h_{i,t} \gamma \Pm \frac{ f_i(\xv_{i,t}) + \mv_{i,t} }{\card{\Pc_t} h_{i,t}} +  \nv_t \nonumber \\
&\stackrel{(b)}{=} \frac{ \gamma \Pm }{\card{\Pc_t}}  \sum_{i \in \mathcal{P}_t} \LB f_i(\xv_{i,t}) + \mv_{i,t} \RB +  \nv_t \nonumber \\
&= \frac{\gamma \Pm}{\card{\Pc_t}} \sum_{i \in \mathcal{P}_t} f_i(\xv_{i,t}) + \frac{\gamma \Pm}{\card{\Pc_t}}  \sum_{i\in\mathcal{P}_t}\mv_{i,t} + \nv_t,
\end{align}
where $(a)$ follows from \cref{eq:gi,eq:yit} and $(b)$ follows from the linearity properties of the projection and scalars.
After receiving $\zv_t$, \gls{CIS} decision function $s(\cdot)$ multiplies the received signal by $\frac{1}{\gamma} \Pm^T$ to recover the desired signal, which is the average of votes, local beliefs or weighted local beliefs:
\begin{align}
    \tilde{s} \LB \zv_t \RB &= \frac{1}{\gamma} \Pm^T \zv_{t} \nonumber \\
    &= \frac{\Pm^T \Pm}{\card{\Pc_t}} \sum_{i \in \Pc_t} f_i (\xv_{i,t}) + \bv_t,
\end{align}
where $\bv_t \sim \Nc \LB \vec{0}, \Pm^T \sigma^2_n + \Pm^T \Pm \gamma^2 \sigma^2_\mathrm{client} /\card{\Pc_t} \RB$ is the accumulated noise due to privacy and wireless channel. Note that when $d \geq k$, we have $\Pm^T \Pm = \vec{I}_k$ due to the orthogonality property; and therefore, the \gls{CIS} reconstructs the average of participating clients' score vectors along with some noise.


Then, the \gls{CIS} applies the $\argmax$ function to decide the most probable class, i.e.,
\begin{equation}
    s (\zv_t ) = \argmax_{j\in[k]} \tilde{s} \LB \zv_t \RB_j.
\end{equation}
\Cref{algor:model} summarizes all the steps introduced in this section. In the next section, we derive the scaling factor $\gamma$.

\begin{algorithm}[tb!]
\caption{OAC-Based Private Collaborative Inference}
\label{algor:model}
\begin{algorithmic}
\Require Trained client model $f_i(\cdot)$ for every client $i$, CIS model $s (\cdot)$, new samples $\xv_{i,t}$ $\forall i \in [n]$ at timestep $t$
\Ensure Index of the decided class
\Function{Private\_Collaborative\_Inference}{}
    \State Let $\Pc_t$ contain each client $i$ with probability $p$, determined via a common seed among clients. 
    \ForAll{client $i \in \Pc_t$ in {\it parallel}}
        \State Client $i$ receives $\xv_{i,t}$
        \State Calculate $f_i(\xv_{i,t})$ \Comment{Client Model}
        \State $g_i(\xv_{i,t}) = f_i(\xv_{i,t}) + \mathcal{N}(\vec{0}, \sigma^{2}/|\Pc_t| \vec{I}_k)$ \Comment{Add noise}
        \State Transmit $\gamma \Pm g_i (\xv_t) / \LB \card{\Pc_t} h_{i,t} \RB$
    \EndFor
    \State $\zv_t = \nv_t +  \gamma \sum_{i\in\Pc_t}g_i(\xv_{i,t})$ \Comment{Air Sum}
    \State CIS receives $\zv_t$
    \State \Return $s(\zv_t)$ \Comment{CIS Model}
\EndFunction
\end{algorithmic}
\end{algorithm}

\subsection{Determining the Scaling Factor}
\label{subsec:power_normalization}
Here, we derive the scaling factor $\gamma$ to normalize the prediction vectors with additional noise introduced due to privacy according to the average power constraint defined in \cref{eq:system_power_constraint}. In our setting, each client $i$ wants to transmit $\gpv_i \triangleq \Pm g_i(\xv_{i,t})$ to the CIS. 
For a clearer notation, we drop the time index for the rest of this subsection. To account for channel gains and to satisfy the average power constraint in \cref{eq:system_power_constraint}, before transmitting $\gpv_i$, we scale it by $\frac{\gamma}{|\mathcal{P}_t|h_i}$,
where $\gamma$ is a scalar chosen to satisfy the average power constraint. In the following, we derive it in detail. For $i\in [n]$, we have,
\begin{align}
\Expect{}{\norm{\yv_i}_2^2} = \Expect{}{\norm{\frac{\gamma}{|\mathcal{P}_t|h_i} \gpv_i}_2^2} \leq P.
\end{align}

We have
\begin{align}
   \Expect{}{\norm{\yv_i}_2^2} &=\frac{\gamma^2}{|\mathcal{P}_t|^2} \Expect{}{\norm{\frac{1}{h_i} \gpv_i}_2^2} \nonumber \\
   &= \frac{\gamma^2}{|\mathcal{P}_t|^2} \sum_{j=1}^{d}\Expect{}{\frac{(\gpv_i)_j^2}{h_i^2}}\nonumber \\
   &\label{eq:pow_const}\stackrel{(a)}{=}
   \frac{\gamma^2}{|\mathcal{P}_t|^2}\sum_{j=1}^{d}\Expect{}{(\gpv_i)_j^2}\Expect{}{\frac{1}{h_i^2}},
\end{align}
where $(a)$ follows from independence between the elements of $\gpv_i$ and $h_i$. Note that each client transmits only if $h_i^2 \geq h_{\min}$. Since $h_i$ is normally distributed with zero mean and variance $\sigma_h^2$, we have
\begin{align}
    \mu_{1/h}\triangleq \Expect{}{\frac{1}{h_i^2}} &= \frac{2}{\sqrt{2\pi}\sigma_h} \int_{\sqrt{h_{\min}}}^\infty \frac{1}{x^2} e^{-\frac{x^2}{2\sigma_h^2}}dx \nonumber \\
    &= \frac{1}{\sqrt{2\pi}\sigma_h} \int_{h_{\min}}^\infty \frac{1}{h} e^{-\frac{h}{2\sigma_h^2}} dh \nonumber  \\
    &= \frac{1}{\sqrt{2\pi}\sigma_h} \int_{h_{\min}/(2\sigma_h^2)}^\infty \frac{1}{h} e^{-h} dh \nonumber \\
    &= \frac{1}{\sqrt{2\pi}\sigma_h} E_1\left(\frac{h_{\min}}{2\sigma_h^2}\right),
%
\end{align}
where $E_1(x) \triangleq \int_x^{\infty} \frac{e^{-t}}{t}dt$. Hence \cref{eq:pow_const} can be written as 
\begin{align}  \frac{\gamma^2}{|\mathcal{P}_t|^2} \mu_{1/h}  \sum_{j=1}^{d}\Expect{}{(\gpv_i)_j^2}= \frac{\gamma^2}{|\mathcal{P}_t|^2} \mu_{1/h} \norm{\gpv_i}_2^2.
\end{align}
Recall that 
$\gpv_i=\Pm g_i(\xv_{i,t})= \Pm \left( f_i(\xv_{i,t}) + \mv_{i,t} \right)$. Since $\mv_{i,t}$ consists of independent Gaussian random variables with variance $\sigma^2_{\mathrm{client}}$, we have 
\begin{equation}
    \frac{\gamma^2}{|\mathcal{P}_t|^2} \mu_{1/h} \norm{\gpv_i}_2^2 = \frac{\gamma^2}{|\mathcal{P}_t|^2} \mu_{1/h} \left( \norm{\Pm  f_i(\xv_{i,t})}_2^2 + d\sigma_{\mathrm{client}}^2 \right).
\end{equation}

To bound $\norm{\Pm f_i(\xv_{i,t})}$, first, let us consider $\Pm \in \RR^{d \times k}$, where $d\geq k$. In this case, $\Pm$ consists of $k$ orthogonal columns. Hence, we have 
\begin{align}
    \norm{\Pm f_i(\xv_{i,t})}_2^2 &= 
    f_i^T(\xv_{i,t})\Pm^T \Pm f_i(\xv_{i,t})\nonumber \\
    &=  f_i^T(\xv_{i,t})\mathbf{I}_k f_i(\xv_{i,t})\nonumber \\      
    &= \norm{f_i(\xv_{i,t})}_2^2. 
\end{align}

Next, we consider the case $d<k$, in which $\Pm$ consists of $d$ orthogonal rows. Let us write $f_i(\xv_{i,t}) = |f|\mathbf{f}_u$, where $\mathbf{f}_u$ is the unit vector pointing to the same direction as $f_i(\xv_{i,t})$. Hence
\begin{align}
    \norm{\Pm f_i(\xv_{i,t})}_2^2 = |f|^2\sum_{i=1}^d \langle\Pm_i,f_u\rangle^2\leq |f|^2 = \norm{f_i(\xv_{i,t})}_2^2.
\end{align}
where $\Pm_i$ is the $i^{th}$ row of $\Pm$. 

Finally, by \cref{eq:power_shift}, the maximum value of $\norm{f_i(\xv_{i,t})}_2^2$ is attained when $\tilde{f}_i ( \xv_{i,t} )$ is exactly 1 at one coordinate and all zeros in the other coordinates. Hence, 
\begin{equation}
    \norm{f_i(\xv_{i,t})}_2^2 \leq \frac{1}{k^2}(k-1) + \left(1-\frac{1}{k}\right)^2 = 1-\frac{1}{k}.
\end{equation}

Hence, for each client, we must satisfy 
\begin{equation}
    \frac{\gamma^2}{|\mathcal{P}_t|^2} \mu_{1/h}\left( 1-\frac{1}{k} + d\sigma_{\mathrm{client}}^2 \right) \leq P,
\end{equation}
which leads to
\begin{equation}
    \gamma =  |\mathcal{P}_t| \sqrt{\frac{P}{ \mu_{1/h}\left( 1-\frac{1}{k} + d\sigma_{\mathrm{client}}^2 \right)}}.
\end{equation}


\section{Privacy Analysis}
\label{sec:privacy-analysis}

In this section, we provide the privacy analysis of the proposed over-the-air ensemble scheme. 
Note that, as discussed in \cref{subsec:ensuring-privacy}, while calculating the DP guarantees, the channel noise $\Vec{n}_t$ is ignored. Thus, our guarantees are upper bounds and the actual privacy guarantees can be stronger due to the channel noise.

We first analyze the case in which all the clients participate. 

\begin{thm}
\label{thm:eps-del-thm}
If all the clients participate in the inference, i.e., $p=1$, then, \cref{algor:model} is $(\varepsilon,\delta)$-DP such that for any $\varepsilon>0$,
\begin{equation}
\label{eq:eps-del-thm}
\delta=\Phi(1/(\sqrt{2}\sigma)-\varepsilon\sigma/\sqrt{2})-e^\varepsilon\Phi(-1/(\sqrt{2}\sigma)-\varepsilon\sigma/\sqrt{2}),    
\end{equation}
where $\Phi$ is the cumulative distribution function of standard normal distribution.
\end{thm}

\begin{proof}
Our theorem is a special case of the following lemma.

\begin{lem}[Theorem 8 in~\cite{balle2018improving}]\label{lem:analytical-GM}
Let $f:\mathbb{L}\to \mathbb{R}^k$ be a function with $||f(\mathcal{L})-f(\mathcal{L}')||_2\leq C$, where $\mathcal{L}$ and $\mathcal{L}'$ are neighboring inputs and $||\cdot||_2$ is $L_2$ norm. A mechanism $M(\mathcal{L})=f(\mathcal{L})+\mathcal{N}(0,\sigma^2\Vec{I}_k)$ is $(\varepsilon,\delta)$-DP if and only if 
\begin{equation}
\label{eq:analytical-GM}
    \Phi\left(C/(2\sigma)-\varepsilon\sigma/C\right)-e^{\varepsilon}\Phi\left(-C/(2\sigma)-\varepsilon\sigma/C\right)\leq \delta.
\end{equation}
\end{lem}

To apply \cref{lem:analytical-GM} directly in our case, first observe that $\gamma$ and $\Pm$ are common to all the clients. Hence, the effective noise is directly added to $\Vec{\tilde{z}}_t \triangleq \sum_{i \in \mathcal{P}_t} f_i(\xv_{i,t})$, and we need to calculate the $L_2$ sensitivity, $C$, of $\Vec{\tilde{z}}_t$. Note that we can also ignore the operation in \cref{eq:power_shift} since it is only a shift operation that does not change the sensitivity calculations.
Consider neighboring sets $\mathcal{L}$ and $\mathcal{L}'$. We denote the noiseless vector received by the CIS by $\Vec{\tilde{z}}_t$ when the set of local models is $\mathcal{L}$, and by $\Vec{\tilde{z}}_t'$ when it is $\mathcal{L}'$. Then,
\begin{equation}
    C = \max_{\Vec{\tilde{z}}_t, \Vec{\tilde{z}}_t' } \left\Vert \Vec{\tilde{z}}_t- \Vec{\tilde{z}}_t' \right\Vert_2 = \max_{\Vec{\tilde{z}}_t \Vec{\tilde{z}}_t' } \left ( \sum_{j=1}^{k} (\Vec{\tilde{z}}_{t,j} - \Vec{\tilde{z}}'_{t,j})^2  \right)^{1/2}.
\end{equation}

We know that $\LB f_i(\xv_{i,t}) \RB _j \in [0,1], \forall j \in [k]$.
Hence, $\Vert \Vec{\tilde{z}}_t- \Vec{\tilde{z}}_t' \Vert_2$ is maximized 
when $\Vec{\tilde{z}}_t-\Vec{\tilde{z}}_t'$ has only two non-zero elements: one is 1 and the other is -1.
Then, $C=\max_{\Vec{\tilde{z}}_t, \Vec{\tilde{z}}_t' } \Vert \Vec{\tilde{z}}_t- \Vec{\tilde{z}}_t' \Vert_2=\sqrt{2}$.

Finally, by substituting $C=\sqrt{2}$ into \cref{eq:analytical-GM}, we obtain \cref{eq:eps-del-thm}.
\end{proof}

Next, we present the amplification effect of client sampling on the privacy guarantees.

\begin{thm}
\label{thm:eps-del-samp-thm}
If each client independently participate in inference with probability $p<1$, then \cref{algor:model} is $(\varepsilon',\delta')$-DP, where, for any $\varepsilon'>0$, 
\begin{multline}
\label{eq:eps-del-samp-thm}
\delta'=\frac{p}{1-(1-p)^n}\Big(\Phi(1/(\sqrt{2}\sigma)-\varepsilon\sigma/\sqrt{2}) \\
-e^{\varepsilon}\Phi(-1/(\sqrt{2}\sigma)-\varepsilon\sigma/\sqrt{2})\Big),
\end{multline}
where $\varepsilon=\log{1+((1-(1-p)^n)/p)(e^{\varepsilon'}-1)}$.
\end{thm}
\begin{proof}
Without loss of generality, let $\mathcal{L}$ and $\mathcal{L'}$ be two neighbouring sets of models differing only in the first client's model, i.e., it is either $f_1$ or $f_1'$. Let us write the output distribution of \cref{algor:model} as a mixture distribution. When the model set is $\mathcal{L}$, we have 
$\mu = (1-\eta)\mu_0 + \eta\mu_1$ and when the model set is $\mathcal{L'}$, we have $\mu' = (1-\eta)\mu_0 + \eta\mu_1'$. In these expressions, $\eta$ is the probability that client 1 is sampled, $\mu_0$ is the output distribution when client 1 is not sampled, $\mu_1$ is the output distribution when client 1 is sampled and the model set is $\mathcal{L}$, while $\mu_1'$ is the output distribution when client 1 is sampled and the model set is $\mathcal{L'}$. Recall that we sample client models each with probability $p$ from $\mathcal{L}$ or $\mathcal{L}'$, and the CIS receives non-zero vectors only when $|\mathcal{P}_t|>0$. Hence, $\eta=\Pr\{\text{Client 1 is sampled}\mid |\mathcal{P}_t|>0\}$, resulting in $\eta=p/(1-(1-p)^n)$ via Bayes' rule.

\begin{lem}[Theorem 1 in~\cite{balle2018privacy}]
\label{lem:privacy_profile} A mechanism $\mathcal{M}$ is $(\varepsilon',\delta')$-DP
if and only if 
\begin{equation}
\sup_{\mathcal{L},\mathcal{L'}}D_{\alpha}(\mathcal{M}(\mathcal{L})||\mathcal{M}(\mathcal{L'}))\leq\delta',    
\end{equation}
where $\alpha=e^{\varepsilon'}$ and $D_{\alpha}(\mu||\mu')\triangleq\int_{Z}\max \{0,d\mu(z)-\alpha d\mu'(z)\}d(z).$
\end{lem}
\Cref{lem:privacy_profile} implies that it is enough to bound $D_{\alpha}(\mu||\mu')$ to provide DP guarantees. For this, we use the relation in \cref{lem:ajc}, which is called \emph{advanced joint convexity} of $D_{\alpha}$.

\begin{lem}[Theorem 2 in~\cite{balle2018privacy}]
\label{lem:ajc}For $\alpha\geq1$,
we have 
\begin{equation}
\label{eq:advanced_joint_convexity}
D_{\alpha'}\left(\mu||\mu'\right)=\eta D_{\alpha}\left(\mu_{1}||(1-\beta)\mu_{0}+\beta\mu_{1}'\right)
\end{equation}
where $\alpha'=1+\eta(\alpha-1)$ and $\beta=\alpha'/\alpha$. 
\end{lem}

We further upper bound \cref{eq:advanced_joint_convexity} via convexity:
\begin{equation}
\label{eq:convexity}
    D_{\alpha'}\left(\mu||\mu'\right) \leq \eta (1-\beta)D_{\alpha}(\mu_1||\mu_0) + \eta \beta D_{\alpha}(\mu_1||\mu_1').
\end{equation}

To bound $D_{\alpha}(\mu_1||\mu_0)$, observe that there exists a coupling between $\mu_1$ and $\mu_0$ as follows. For $\mu_0$, to guarantee $|\mathcal{P}_t|>0$, let us first sample exactly one client $c$ other than client 1 since we know that client 1 is not sampled. Then, we apply Poisson sampling on the remaining set, i.e., $[n]\setminus\{c,1\}$, to determine the other participating clients. For $\mu_1$, assume that we have the same realization of Poisson sampling on $[n]\setminus \{c,1\}$ as in $\mu_0$. Further, by definition of $\mu_1$, client 1 is also sampled. Hence, $\mu_1$ and $\mu_0$ can be seen as output distributions of \cref{algor:model} such that the input client sets differ in only one element. Hence, $D_{\alpha}(\mu_1||\mu_0)\leq \delta$ due to \cref{thm:eps-del-thm}. Similarly, to bound $D_{\alpha}(\mu_1,\mu_1')$, a coupling exists between $\mu_1$ and $\mu_1'$ such that user 1 is sampled and $f_1$ and $f_1'$ are the models in user 1, for $\mu_1$ and $\mu_1'$, respectively. To determine the other participating clients, the same realization of Poisson sampling on $[n]\setminus\{1\}$ is applied in both $\mu_1$ and $\mu_1'$. Since the input client sets also differ in one element, in this case, due to \cref{thm:eps-del-thm}, we have $D_{\alpha}(\mu_1,\mu_1')\leq \delta$. If we insert the bounds on $D_{\alpha}(\mu_1,\mu_0)$ and $D_{\alpha}(\mu_1,\mu_1')$ into \cref{eq:convexity}, we obtain $D_{\alpha'}(\mu||\mu')\leq \eta \delta$, from which \cref{eq:eps-del-samp-thm} follows. The expression for $\varepsilon$ can be directly derived from the expression $\alpha'=1+\eta(\alpha-1)$.
\end{proof}


\section{Numerical Results}
\label{sec:numerical_results}
In this section, we present our experimental setup, implementation details and simulation results.


\subsection{The Datasets}
We employ eight different multi-class classification datasets to demonstrate the effectiveness of our framework: CIFAR-10, CIFAR-100, FashionMNIST, Food101, OxfordPets, Emotion, Imdb, and MultiviewPets. CIFAR-10 is a multi-class image classification dataset containing \num{50000} training images, \num{10000} test images, and \num{10} target classes~\cite{krizhevsky2009learning}.
CIFAR-100 is the same dataset with finer granularity, i.e., it is partitioned into \num{100} target classes~\cite{krizhevsky2009learning}. FashionMNIST is a multi-class fashion image classification dataset with \num{60000} training images, \num{10000} test images, and 10 target classes~\cite{xiao2017/online}. Food101 is a dish classification dataset with \num{75750} training and \num{25250} test images, and \num{101} target classes~\cite{bossard14}. OxfordPets (Oxford-IIIT Pet) is a multi-class pet classification dataset with \num{37} target classes, \num{3680} training and \num{3669} test images~\cite{catsdogs}. Emotion is an English Twitter sentiment classification dataset with \num{16000} training and \num{2000} test texts, and \num{6} target classes~\cite{saravia-etal-2018-carer}. Imdb is a binary sentiment classification dataset with \num{25000} training texts, and \num{25000} test texts~\cite{maas-EtAl:2011:ACL-HLT2011}.

For all the datasets, we use predefined training and testing sets, except that we split 10\% of the training set as the validation set and only use the remaining 90\% for training. Note that the training splits are further split disjointly across clients as described in \cref{sec:system_model} except the multi-view datasets. We have generated MultiViewPets datasets using OxfordPets images and labels following the common practice of simulating different camera angles. Instead of splitting the dataset disjointly, we crop regions of the same images and assign a region to each client. We crop \num{20} regions (via \num{4} vertical and \num{5} horizontal lines) with \num{50}\% overlap between neighboring regions.

\begin{remark}
    Throughout the paper, we assume that our privacy guarantees are local-model-level, which implies the same DP guarantees for the local datasets. However, in the case of multi-view classification on the MultiViewPets dataset, this is no longer valid since some parts of the data samples located in a client also exist in other clients due to the nature of the multi-view setting. This results in weaker \gls{DP} guarantees at the local dataset level. Nevertheless, our local-model-level privacy guarantee and all of its benefits still hold.
\end{remark}

\subsection{Experimental Setup}
We repeat all the experiments with 5 different random seeds and report the average results.  We randomly split the training data among the clients equally. We consider $n=20$ clients with a participation probability of $p=1.0$ and a channel \gls{SNR} of \num{0} \si{\decibel}, except when they are changed gradually in \cref{sec:conditions}. 

We compute and report macro-averaged F1 (Macro-F1) scores by averaging per-class F1 scores, which is a commonly employed metric in multi-class classification since it treats each class with equal significance, ensuring a balanced evaluation across all classes.


\subsection{Implementation Details}
We employ PyTorch for all the experiments, torchvision for vision experiments and Huggingface transformers for text experiments~\cite{pytorch,marcel2010torchvision,wolf-etal-2020-transformers}. For image datasets, we use MobileNetV3-Large~\cite{howard2019searching} except we change its final layer to make it compatible with the target number of classes. Instead of training from scratch, we fine-tune a pre-trained model~\cite{marcel2010torchvision} for \num{50} epochs. Similarly, for text datasets, we fine-tune the DistilBERT-base-uncased model~\cite{sanh2019distilbert} for \num{50} epochs. To make the sizes of the images compatible with our MobileNetV3 network, we interpolate them to $224\times224$ images following PyTorch's pretraining recipe. Since the network accepts inputs with three channels, we convert grayscale datasets by repeating the grayscale image for all RGB channels. 
We employ \gls{SGD} optimizer with a batch size of \num{128}, momentum \num{0.9}, weight decay \num{5e-4}, learning rate \num{0.05} and cosine annealing scheduling method~\cite{loshchilov2016sgdr}. We employ cross-entropy loss for training all local models. We also utilize label smoothing factor $0.1$ for training vision models to better calibrate the beliefs~\cite{szegedy2015rethinking}.

\subsection{Comparison with the Baselines}
\label{sec:comparison_baselines}

In \cref{tab:comparison}, we compare the proposed OAC-based methods with the {\it Best Client} model and the ensemble methods with orthogonal transmission in terms of their Macro-F1 scores. We choose the model with the highest Macro-F1 score on the same validation set as the {\it Best Client} model. For fairness, the client having the best model transmits its inference over $k$ channels. In orthogonal methods, denoted by {\it -Orth} suffix, all the devices transmit their inferences via different channels, i.e., $\card{\mathcal{P}_t} \times k$ channels in total.

\begin{table*}[t]
    \centering
    \caption{Comparison of our method with the baselines in terms of Macro-F1 scores}
    \resizebox{\textwidth}{!}{%
    \begin{tabular}{clcccccccc}
    \toprule
    $\varepsilon$ & Method & CIFAR-10 & CIFAR-100 & FashionMnist & Food101 & OxfordPets & Emotion & Imdb & MultiViewPets\\ \midrule
\multirow{7}{*}{$\infty$} & Best Client & $83.16  {\scriptstyle \pm 0.29}$ & $59.49  {\scriptstyle \pm 0.67}$ & $84.24  {\scriptstyle \pm 0.14}$ & $58.42  {\scriptstyle \pm 0.23}$ & $58.70  {\scriptstyle \pm 2.07}$ & $71.03  {\scriptstyle \pm 1.87}$ & $82.90  {\scriptstyle \pm 0.19}$ & $83.33  {\scriptstyle \pm 0.41}$ \\ 
  & BA-Orth & $91.57  {\scriptstyle \pm 0.04}$ & $73.02  {\scriptstyle \pm 0.35}$ & $91.87  {\scriptstyle \pm 0.07}$ & $71.25  {\scriptstyle \pm 0.12}$ & $82.87  {\scriptstyle \pm 0.57}$ & $83.80  {\scriptstyle \pm 0.52}$ & $90.86  {\scriptstyle \pm 0.07}$ & $87.46  {\scriptstyle \pm 0.24}$ \\ 
  & WBA-Orth & $91.56  {\scriptstyle \pm 0.04}$ & $73.02  {\scriptstyle \pm 0.34}$ & $91.86  {\scriptstyle \pm 0.08}$ & $71.24  {\scriptstyle \pm 0.12}$ & $82.88  {\scriptstyle \pm 0.64}$ & $83.84  {\scriptstyle \pm 0.50}$ & $90.86  {\scriptstyle \pm 0.07}$ & $87.43  {\scriptstyle \pm 0.23}$ \\ 
  & MV-Orth & $91.55  {\scriptstyle \pm 0.09}$ & $72.12  {\scriptstyle \pm 0.26}$ & $91.84  {\scriptstyle \pm 0.04}$ & $70.75  {\scriptstyle \pm 0.17}$ & $82.03  {\scriptstyle \pm 0.70}$ & $83.69  {\scriptstyle \pm 0.58}$ & $90.83  {\scriptstyle \pm 0.06}$ & $86.58  {\scriptstyle \pm 0.22}$ \\ 
  & BA-OAC & $91.79  {\scriptstyle \pm 0.10}$ & $\mathbf{73.25  {\scriptstyle \pm 0.36}}$ & $\mathbf{91.99  {\scriptstyle \pm 0.12}}$ & $\mathbf{71.74  {\scriptstyle \pm 0.10}}$ & $\mathbf{83.55  {\scriptstyle \pm 0.61}}$ & $83.85  {\scriptstyle \pm 0.63}$ & $\mathbf{91.00  {\scriptstyle \pm 0.05}}$ & $\mathbf{87.77  {\scriptstyle \pm 0.26}}$ \\ 
  & WBA-OAC & $\mathbf{91.80  {\scriptstyle \pm 0.10}}$ & $73.24  {\scriptstyle \pm 0.35}$ & $\mathbf{91.99  {\scriptstyle \pm 0.12}}$ & $71.72  {\scriptstyle \pm 0.09}$ & $83.51  {\scriptstyle \pm 0.63}$ & $\mathbf{83.86  {\scriptstyle \pm 0.64}}$ & $\mathbf{91.00  {\scriptstyle \pm 0.05}}$ & $\mathbf{87.77  {\scriptstyle \pm 0.24}}$ \\ 
  & MV-OAC & $91.66  {\scriptstyle \pm 0.06}$ & $72.04  {\scriptstyle \pm 0.28}$ & $91.88  {\scriptstyle \pm 0.08}$ & $70.80  {\scriptstyle \pm 0.14}$ & $82.33  {\scriptstyle \pm 0.54}$ & $83.60  {\scriptstyle \pm 0.51}$ & $90.99  {\scriptstyle \pm 0.06}$ & $86.79  {\scriptstyle \pm 0.20}$ \\ 
 \midrule  
 \multirow{7}{*}{$5$} & Best Client & $18.37  {\scriptstyle \pm 0.13}$ & $2.30  {\scriptstyle \pm 0.10}$ & $18.84  {\scriptstyle \pm 0.16}$ & $2.35  {\scriptstyle \pm 0.23}$ & $5.42  {\scriptstyle \pm 0.24}$ & $22.99  {\scriptstyle \pm 0.79}$ & $60.68  {\scriptstyle \pm 0.16}$ & $6.40  {\scriptstyle \pm 0.20}$ \\  
  & BA-Orth & $53.36  {\scriptstyle \pm 0.34}$ & $6.65  {\scriptstyle \pm 0.19}$ & $57.38  {\scriptstyle \pm 0.33}$ & $6.55  {\scriptstyle \pm 0.25}$ & $9.17  {\scriptstyle \pm 0.73}$ & $60.17  {\scriptstyle \pm 1.18}$ & $84.64  {\scriptstyle \pm 0.20}$ & $19.35  {\scriptstyle \pm 0.35}$ \\  
  & WBA-Orth & $53.36  {\scriptstyle \pm 0.33}$ & $6.66  {\scriptstyle \pm 0.20}$ & $57.36  {\scriptstyle \pm 0.35}$ & $6.55  {\scriptstyle \pm 0.25}$ & $9.20  {\scriptstyle \pm 0.76}$ & $60.16  {\scriptstyle \pm 1.16}$ & $84.64  {\scriptstyle \pm 0.20}$ & $19.32  {\scriptstyle \pm 0.37}$ \\  
  & MV-Orth & $65.08  {\scriptstyle \pm 0.46}$ & $19.17  {\scriptstyle \pm 0.25}$ & $66.32  {\scriptstyle \pm 0.57}$ & $19.61  {\scriptstyle \pm 0.38}$ & $25.16  {\scriptstyle \pm 0.69}$ & $60.16  {\scriptstyle \pm 1.32}$ & $84.62  {\scriptstyle \pm 0.21}$ & $35.59  {\scriptstyle \pm 0.35}$ \\  
  & BA-OAC & $91.40  {\scriptstyle \pm 0.15}$ & $56.68  {\scriptstyle \pm 0.24}$ & $\mathbf{91.89  {\scriptstyle \pm 0.13}}$ & $50.85  {\scriptstyle \pm 0.26}$ & $65.89  {\scriptstyle \pm 1.47}$ & $83.69  {\scriptstyle \pm 0.41}$ & $90.96  {\scriptstyle \pm 0.07}$ & $83.22  {\scriptstyle \pm 0.20}$ \\  
  & WBA-OAC & $91.39  {\scriptstyle \pm 0.14}$ & $56.67  {\scriptstyle \pm 0.24}$ & $91.87  {\scriptstyle \pm 0.11}$ & $50.84  {\scriptstyle \pm 0.27}$ & $65.90  {\scriptstyle \pm 1.52}$ & $\mathbf{83.83  {\scriptstyle \pm 0.40}}$ & $90.96  {\scriptstyle \pm 0.07}$ & $83.21  {\scriptstyle \pm 0.20}$ \\  
  & MV-OAC & $\mathbf{91.60  {\scriptstyle \pm 0.07}}$ & $\mathbf{71.21  {\scriptstyle \pm 0.21}}$ & $91.80  {\scriptstyle \pm 0.07}$ & $\mathbf{69.98  {\scriptstyle \pm 0.08}}$ & $\mathbf{81.17  {\scriptstyle \pm 0.77}}$ & $83.51  {\scriptstyle \pm 0.86}$ & $\mathbf{90.97  {\scriptstyle \pm 0.05}}$ & $\mathbf{86.30  {\scriptstyle \pm 0.21}}$ \\  
 \midrule  
 \multirow{7}{*}{$1$} & Best Client & $11.55  {\scriptstyle \pm 0.12}$ & $1.21  {\scriptstyle \pm 0.13}$ & $11.71  {\scriptstyle \pm 0.23}$ & $1.26  {\scriptstyle \pm 0.13}$ & $3.21  {\scriptstyle \pm 0.20}$ & $16.42  {\scriptstyle \pm 0.99}$ & $52.65  {\scriptstyle \pm 0.18}$ & $3.38  {\scriptstyle \pm 0.18}$ \\  
  & BA-Orth & $17.21  {\scriptstyle \pm 0.26}$ & $1.56  {\scriptstyle \pm 0.11}$ & $17.81  {\scriptstyle \pm 0.15}$ & $1.52  {\scriptstyle \pm 0.13}$ & $3.54  {\scriptstyle \pm 0.30}$ & $23.55  {\scriptstyle \pm 1.24}$ & $61.15  {\scriptstyle \pm 0.16}$ & $4.34  {\scriptstyle \pm 0.22}$ \\  
  & WBA-Orth & $17.21  {\scriptstyle \pm 0.27}$ & $1.56  {\scriptstyle \pm 0.11}$ & $17.81  {\scriptstyle \pm 0.16}$ & $1.52  {\scriptstyle \pm 0.13}$ & $3.56  {\scriptstyle \pm 0.30}$ & $23.55  {\scriptstyle \pm 1.24}$ & $61.15  {\scriptstyle \pm 0.16}$ & $4.34  {\scriptstyle \pm 0.22}$ \\  
  & MV-Orth & $19.31  {\scriptstyle \pm 0.14}$ & $2.22  {\scriptstyle \pm 0.19}$ & $19.44  {\scriptstyle \pm 0.16}$ & $2.26  {\scriptstyle \pm 0.11}$ & $4.92  {\scriptstyle \pm 0.35}$ & $23.60  {\scriptstyle \pm 1.23}$ & $61.14  {\scriptstyle \pm 0.18}$ & $5.76  {\scriptstyle \pm 0.22}$ \\  
  & BA-OAC & $71.14  {\scriptstyle \pm 0.37}$ & $12.72  {\scriptstyle \pm 0.26}$ & $76.07  {\scriptstyle \pm 0.36}$ & $12.14  {\scriptstyle \pm 0.15}$ & $13.80  {\scriptstyle \pm 0.99}$ & $\mathbf{74.80  {\scriptstyle \pm 1.08}}$ & $89.14  {\scriptstyle \pm 0.13}$ & $32.39  {\scriptstyle \pm 0.95}$ \\  
  & WBA-OAC & $71.14  {\scriptstyle \pm 0.37}$ & $12.72  {\scriptstyle \pm 0.26}$ & $76.04  {\scriptstyle \pm 0.37}$ & $12.14  {\scriptstyle \pm 0.15}$ & $13.83  {\scriptstyle \pm 0.99}$ & $74.78  {\scriptstyle \pm 1.05}$ & $89.14  {\scriptstyle \pm 0.13}$ & $32.39  {\scriptstyle \pm 0.93}$ \\  
  & MV-OAC & $\mathbf{82.43  {\scriptstyle \pm 0.33}}$ & $\mathbf{36.24  {\scriptstyle \pm 0.58}}$ & $\mathbf{83.73  {\scriptstyle \pm 0.19}}$ & $\mathbf{36.66  {\scriptstyle \pm 0.19}}$ & $\mathbf{40.95  {\scriptstyle \pm 1.60}}$ & $74.66  {\scriptstyle \pm 0.94}$ & $\mathbf{89.19  {\scriptstyle \pm 0.12}}$ & $\mathbf{56.40  {\scriptstyle \pm 0.37}}$ \\  
 \bottomrule
            \end{tabular}}
            \label{tab:comparison}
            \end{table*}

\begin{figure*}[!t]
\centering
\foreach \plottype/\plotname/\epsval in {non_private/Non-private/\infty,weak_private/Weak private/5,private/Private/1}{    \begin{tikzpicture}[scale=0.55]
	\pgfplotstableread[col sep=comma]{results/comparison_\plottype.csv}\csvdata
	\begin{axis}[
            title = {$\varepsilon = \epsval$},
		boxplot/draw direction = y,
		axis x line* = bottom,
		axis y line = left,
		xtick = {1, 2, 3, 4, 5, 6, 7},
		xticklabel style = {align=center, rotate=60},
		xticklabels = {Best Client,BA-Orth,WBA-Orth,MV-Orth,MV-OAC,BA-OAC,WBA-OAC},
		xtick style = {draw=none}, 
		ylabel = {Macro-F1},
            ymin=0.0,
            ymax=1.05,
            ytick={0,0.5,1},
            enlarge y limits=0,
            enlarge x limits=0.05,
	]
		\foreach \n in {0,...,6} {
			\addplot+[boxplot, fill, mark=none, draw=black, solid] table[y index=\n] {\csvdata};
		}
	\end{axis}
\end{tikzpicture}
}
\caption{Box plots for the Macro-F1 score for all the datasets in the case of non-private ($\varepsilon=\infty$, left), moderately private, ($\varepsilon=5$, middle), strongly private ($\varepsilon=1$, right) scenarios.} 
\label{fig:comparison_boxplots}
\end{figure*}
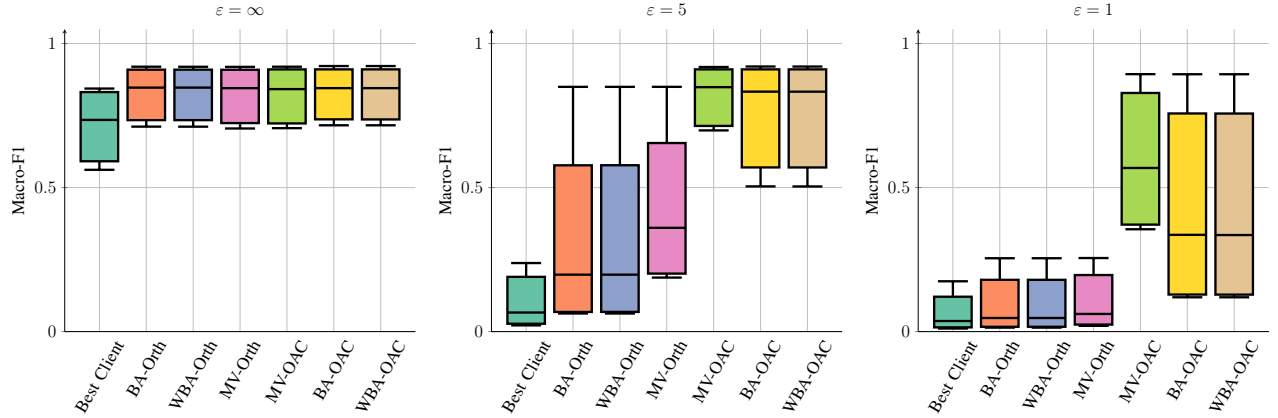

We observe that, compared to the {\it Best Client} model, ensemble methods significantly improve the test scores, especially in the private setting. Moreover, while orthogonal and OAC-based methods perform competitively in the non-private setting, when privacy is involved, the {\it Best Client} model and orthogonal methods perform near-random, and significantly worse than the OAC-based methods. Note that orthogonal methods need $\card{\mathcal{P}_t} \times k$ channel uses, whereas OAC-based methods need only $k$ channel uses; yet, OAC-based methods outperform orthogonal ones in the private setting. In all the privacy settings, i.e., $\varepsilon \in \{1,5,\infty\}$, {\it Weighted Belief Averaging} and standard (unweighted) {\it Belief Averaging} methods perform very close for both orthogonal and \gls{OAC}-based variants, so we refer to both  {\it Belief Averaging} and {\it Weighted Belief Averaging} via {\it Belief Averaging} term in the rest of this subsection.

Previous studies suggest that ensembling via {\it Belief Averaging} generally performs better than {\it Majority Voting}~\cite{kuncheva2002theoretical,wang2020averaging}. Our non-private results also support this argument as beliefs contain more information compared to conveying local decisions. However, interestingly, when $\varepsilon=1$ or $\varepsilon=5$, {\it Majority Voting} outperforms {\it Belief Averaging} for both orthogonal and OAC-based settings. This can be explained by the fact that the increasing noise levels result in relatively unreliable beliefs, since the individual values of beliefs are smaller, and thus more sensitive to the noise added for privacy.

There is no significant winner between the {\it Majority Voting} and {\it Belief Averaging} schemes, especially in the non-private case, for both orthogonal and \gls{OAC} cases. This situation contradicts the fact that averaging, which is equivalent to our {\it Belief Averaging}, is the clear winner in previous studies~\cite{kuncheva2002theoretical,wang2020averaging}, in which there is no noise added to the predictions due to privacy or wireless channel. This is probably caused by the channel and privacy noise terms, which result in the situation that more precise inferences per model result in a better ensemble. This is also supported by our results, which show that as noise increases, the relative performance of {\it Majority Voting} increases compared to {\it Belief Averaging} counterparts.

Moreover, as the number of classes decreases, the difference between {\it Belief Averaging} and {\it Majority Voting} decreases, which is probably due to the spread of the beliefs among classes. This can be observed through the difference between {\it Belief Averaging} and {\it Majority Voting} experiments on Imdb (\num{2} classes), Emotion (\num{6} classes), CIFAR-10 (\num{10} classes) and CIFAR-100 (\num{100} classes) datasets.

In the multi-view setting on the MultiViewPets dataset, the methods behave similarly in terms of Macro-F1 score compared to the ensemble setting. \Gls{OAC}-based methods outperform all of the orthogonal methods and {\it Best Client} method. In the non-private setting, {\it Belief Averaging}-based methods outperform {\it Majority Voting} while in weak-private and private settings {\it Majority Voting} outperforms {\it Belief Averaging}-based methods. Similar to the ensemble setting, {\it Best Client} performs the worst among all the methods in the same privacy setting.

\Cref{fig:comparison_boxplots} shows the boxplots of the evaluated methods over all the datasets for different privacy levels. In addition to the abovementioned points, as shown in the box plots, we observe an increase in the standard deviations of all the methods over different datasets as we decrease $\varepsilon$. This increase is caused by the privacy noise added as we decrease $\varepsilon$. These box plots also validate our above-mentioned claims.



\subsection{Comparison with Varying Conditions}
\label{sec:conditions}
\Cref{fig:conditions_label,fig:conditions_belief,fig:conditions_weighted_belief} show the performance on the CIFAR-10 dataset for varying channel \gls{SNR}, $p$, $d$ and $\varepsilon$ values of {\it Majority Voting}, {\it Belief Averaging} and {\it Weighted Belief Averaging} fusion methods, respectively.

These three fusion methods perform similarly with respect to the evaluated channel \gls{SNR}, $p$ and $d$. The left figures show that the performance of the methods slightly increases with the channel \gls{SNR}, especially for \gls{SNR} values below 2 dB. In the middle figures, we observe that higher $p$ improves the performance significantly in the private settings. Although lower $p$ has a privacy amplification effect which decreases the additional noise variance required to attain the desired privacy level, we observe that its privacy amplification effect is not as significant as the impact of fewer client participation on the inference performance. 
In the non-private setting $(\varepsilon=\infty)$, having higher participation also helps to get a higher macro-F1 score, but not as much as in the private setting. These plots also show that the private setting is more sensitive to varying conditions for both $p$ and channel \gls{SNR}. The right hand side figures clearly demonstrate that the result improves as the projection dimension $d$ increases until $d=k$ and reaches its peak at $d=k$. Further increase in $d$ does not benefit the result since we consider average power constraint in this work, which keeps the noise ratio fixed with respect to the signal.

\subsection{Ablation Study of Privacy and Projection Methods}
\label{sec:ablation_privacy_projection}

\foreach \plottype/\plotname/\fulldesc in {label/MV/MV-OAC and MV-Orth ,belief/BA/BA-OAC and BA-Orth ,weighted_belief/WBA/WBA-OAC and WBA-Orth }{
\begin{figure*}[!t]
    \centering
\begin{tikzpicture}[scale=0.6]
\begin{axis}[
xlabel={$\mathrm{SNR}$ (\si{\decibel})},
ylabel={Macro-F1},
legend pos={south east},
legend to name=conditions_legend_\plottype,
legend columns=-1,
error bars/y dir=both,
error bars/y explicit,
legend style={font=\footnotesize},
ymin=0.0,
ymax=1.00,
ytick={0,0.2,0.4,0.6,0.8,1},
]
\foreach \method in {\plottype_oac_nonprivate,\plottype_orthogonal_nonprivate,\plottype_oac_weakprivate,\plottype_orthogonal_weakprivate,\plottype_oac_private,\plottype_orthogonal_private}{
    \addplot table[x=snr, y=\method, y error=\method_std, col sep=comma]{results/conditions_snr_vs_macro_f1.csv};
}
\legend{\plotname-OAC $(\varepsilon=\infty)$,\plotname-Orth $(\varepsilon=\infty)$,\plotname-OAC $(\varepsilon=5)$,\plotname-Orth $(\varepsilon=5)$,\plotname-OAC $(\varepsilon=1)$,\plotname-Orth $(\varepsilon=1)$}
\end{axis}
\end{tikzpicture}
\begin{tikzpicture}[scale=0.6]
\begin{axis}[
xlabel={Participation probability $p$},
ylabel={Macro-F1},
legend pos={south east},
error bars/y dir=both,
error bars/y explicit,
ymin=0.0,
ymax=1.00,
ytick={0,0.2,0.4,0.6,0.8,1},
]
\foreach \method in {\plottype_oac_nonprivate,\plottype_orthogonal_nonprivate,\plottype_oac_weakprivate,\plottype_orthogonal_weakprivate,\plottype_oac_private,\plottype_orthogonal_private}{
    \addplot table[x=p, y=\method, y error=\method_std, col sep=comma]{results/conditions_p_vs_macro_f1.csv};
}
\end{axis}
\end{tikzpicture}
\begin{tikzpicture}[scale=0.6]
\begin{axis}[
xlabel={Projection dimensions $d$},
ylabel={Macro-F1},
legend pos={south east},
error bars/y dir=both,
error bars/y explicit,
ymin=0.0,
ymax=1.00,
ytick={0,0.2,0.4,0.6,0.8,1},
]
\foreach \method in {\plottype_oac_nonprivate,\plottype_orthogonal_nonprivate,\plottype_oac_weakprivate,\plottype_orthogonal_weakprivate,\plottype_oac_private,\plottype_orthogonal_private}{
    \addplot table[x=num_dims, y=\method, y error=\method_std, col sep=comma]{results/conditions_num_dims_vs_macro_f1.csv};
}
\end{axis}
\end{tikzpicture}
\resizebox{\textwidth}{!}{
    \ref*{conditions_legend_\plottype}
}
\caption{Comparison of \fulldesc with different \gls{DP} levels as a function of channel SNR (left), participation probability $p$ (middle) and projection dimensions $d$ (right) on the validation split of CIFAR-10 dataset.}
\label{fig:conditions_\plottype}
\end{figure*}
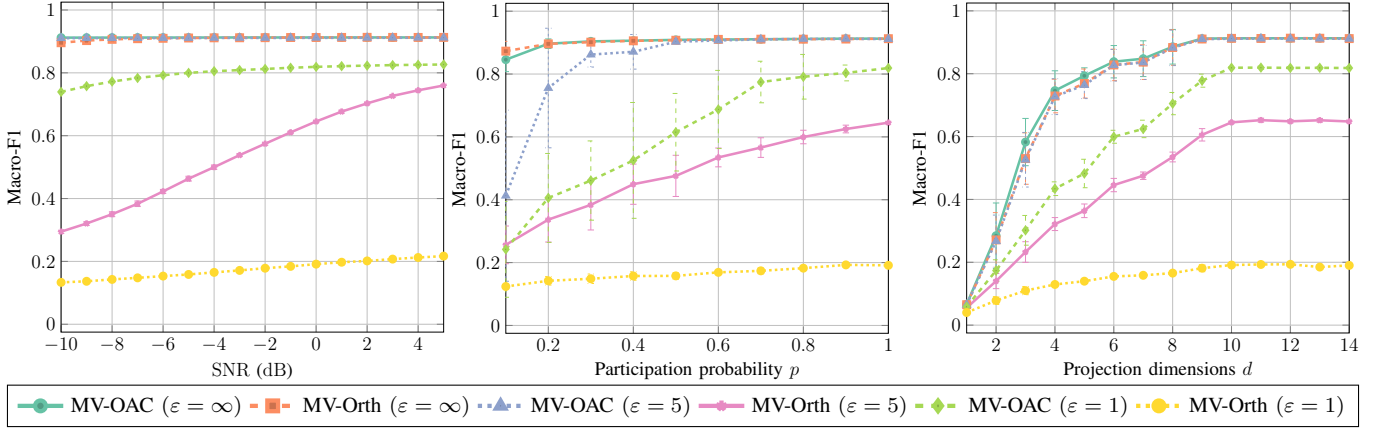
}

Here, we perform an ablation study evaluating different projection methods for private and non-private cases, which is presented in \cref{tab:ablation_projection}. We employ validation set to prevent leakage from the test set by exhaustive search for hyperparameters.

Again, \gls{OAC}-based methods outperform their orthogonal counterparts in all our experiments. Both Gaussian and Rademacher projections are among the most common in the compressed sensing literature. In Gaussian projection, we sample $\Pm$ from random normal distribution, whereas in Rademacher projection, we sample $\Pm$ from Rademacher distribution that consists of $-1$ and $1$ with equal probability. 
Orthogonal projection clearly outperforms both Gaussian and Rademacher projections for all the evaluated $\varepsilon$ values, and performs very closely to identity projection, while having the additional benefit of expanding/shrinking the bandwidth.

In the private setting, we evaluate the projection of private labels, in which we add the privacy noise to the beliefs $f_i(\xv_{i,t})$, i.e., before projection. These results are denoted by appending "of labels" to the method name, and they perform better than the private projection methods that add privacy noise after projection, which are described in~\cite{li2023differential}. This is because \gls{DP} accounts for privacy noise in the worst-case, so although projection has some randomization and distortion effect on the signal, it still tries to protect that part, and this results in worse performance.

Among all these results, orthogonal projection of labels, which is described in \cref{sec:methodology_transmission}, performs the best, and is very close to the identity projection.

As a fully digital baseline, in \cref{tab:ablation_projection}, we also report the results of \gls{RR}~\cite{warner1965randomized} as a private inference method. Since \gls{RR} is a discrete scheme, we only employ it for {\it Majority Voting}. In this technique, each client conveys its true prediction with probability $\frac{e^{\varepsilon}}{e^{\varepsilon}+k}$ and tells otherwise a lie by uniform randomly picking a wrong label. For {\it \gls{MV-OAC}}, we further benefit from shuffling although its effect is limited for $n=20$ clients~\cite{feldman2022hiding}. We note that the \gls{MAC} acts as a shuffler since signals lose the order via summation. Since \gls{RR} works for digital schemes, and it does not act as additive noise, the summation of randomized predictions from the clients does not improve the overall privacy even if we use OAC. 
\begin{table}[htbp!]
    \centering
    \caption{Ablation study of different projection methods on the validation split of CIFAR-10}
\resizebox{\columnwidth}{!}{%
        \begin{tabular}{llcc}
    \toprule
    $\varepsilon$ & Projection & MV-Orth (\%) & MV-OAC (\%)\\ \midrule
 \multirow{4}{*}{$\infty$}  & Identity Projection & $91.15 {\scriptstyle \pm 0.20}$ & $\mathbf{91.33 {\scriptstyle \pm 0.17}}$\\
 & Gaussian Projection & $68.36 {\scriptstyle \pm 22.99}$ & $87.15 {\scriptstyle \pm 8.78}$\\
 & Rademacher Projection & $34.17 {\scriptstyle \pm 39.50}$ & $39.83 {\scriptstyle \pm 46.29}$\\
 & Orthogonal Projection & $91.16 {\scriptstyle \pm 0.16}$ & $91.31 {\scriptstyle \pm 0.21}$\\
 \midrule
 \multirow{7}{*}{$1$}  & Identity Projection & $19.07 {\scriptstyle \pm 0.94}$ & $\mathbf{82.08 {\scriptstyle \pm 0.51}}$\\
 & Gaussian Projection & $8.41 {\scriptstyle \pm 1.89}$ & $20.03 {\scriptstyle \pm 7.84}$\\
 & Gaussian Projection of Labels & $9.84 {\scriptstyle \pm 2.82}$ & $58.01 {\scriptstyle \pm 18.89}$\\
 & Rademacher Projection & $7.15 {\scriptstyle \pm 2.81}$ & $13.05 {\scriptstyle \pm 11.16}$\\
 & Rademacher Projection of Labels & $7.66 {\scriptstyle \pm 3.69}$ & $29.33 {\scriptstyle \pm 32.54}$\\
 & Orthogonal Projection & $19.15 {\scriptstyle \pm 0.64}$ & $81.87 {\scriptstyle \pm 0.13}$\\
 & Orthogonal Projection of Labels & $19.13 {\scriptstyle \pm 0.58}$ & $81.91 {\scriptstyle \pm 0.31}$\\
 & Orthogonal Projection of RR & $43.70 {\scriptstyle \pm 0.98}$ & $52.81 {\scriptstyle \pm 0.52}$\\

\bottomrule
    \end{tabular}
}
    \label{tab:ablation_projection}
    \end{table}

\subsection{Analysis of Privacy Quantities}
In this section, we show the quantitative values of standard deviations for the \gls{DP} noise to be added in different scenarios. \Cref{fig:privacy_quantities} illustrates the effect of the number of users $n$ and participation probability $p$ on the privacy noise. The amount of privacy noise, measured by its variance, decreases as $\varepsilon$ increases. This is expected because higher $\varepsilon$ corresponds to lower privacy, which requires less noise to be added. As participation probability $p$ increases, we need to add more noise in total to guarantee the same level of \gls{DP} since partial participation is another source of randomness. As the number of clients $n$ increases, we need to add less noise per client to guarantee the same level of \gls{DP} since noise is distributed among the participating clients.
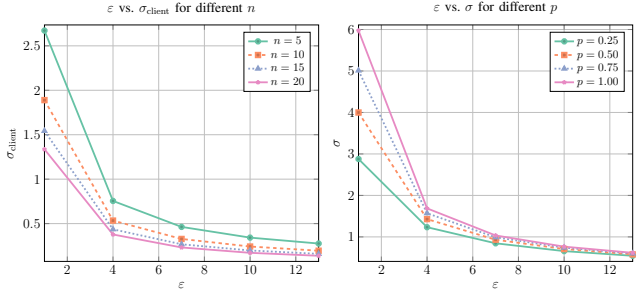
\begin{figure}[!t]
    \centering
    \begin{tikzpicture}[scale=0.43]
\begin{axis}[
title={$\varepsilon$ vs. $\sigma_\mathrm{client}$ for different $n$},
xlabel={$\varepsilon$},
ylabel={$\sigma_\mathrm{client}$},
legend pos={north east},
]
\addplot table[x=eps, y=sigma_n5, col sep=comma]{results/privacy_eps_vs_sigma_by_n.csv};
\addplot table[x=eps, y=sigma_n10, col sep=comma]{results/privacy_eps_vs_sigma_by_n.csv};
\addplot table[x=eps, y=sigma_n15, col sep=comma]{results/privacy_eps_vs_sigma_by_n.csv};
\addplot table[x=eps, y=sigma_n20, col sep=comma]{results/privacy_eps_vs_sigma_by_n.csv};
\legend{$n=5$,$n=10$,$n=15$,$n=20$}
\end{axis}
\end{tikzpicture}
\begin{tikzpicture}[scale=0.43]
\begin{axis}[
title={$\varepsilon$ vs. $\sigma$ for different $p$},
xlabel={$\varepsilon$},
ylabel={$\sigma$},
legend pos={north east},
]
\addplot table[x=eps, y=sigma_p0.25, col sep=comma]{results/privacy_eps_vs_sigma_by_p.csv};
\addplot table[x=eps, y=sigma_p0.5, col sep=comma]{results/privacy_eps_vs_sigma_by_p.csv};
\addplot table[x=eps, y=sigma_p0.75, col sep=comma]{results/privacy_eps_vs_sigma_by_p.csv};
\addplot table[x=eps, y=sigma_p1.0, col sep=comma]{results/privacy_eps_vs_sigma_by_p.csv};
\legend{$p=0.25$,$p=0.50$,$p=0.75$,$p=1.00$}
\end{axis}
\end{tikzpicture}
\caption{The standard deviation of noise for different system and privacy conditions.}
\label{fig:privacy_quantities}
\end{figure}

\subsection{Analysis of Scalability}
\begin{table*}[t]
    \centering
    \caption{Scalability analysis of the introduced methods on CIFAR-10 in terms of Macro-F1}
    \begin{tabular}{clccccc}
    \toprule
    $\varepsilon$ & Method & 5 Users & 15 Users & 20 Users & 50 Users & 100 Users\\ \midrule
\multirow{7}{*}{$\infty$} & Best Client & $86.35  {\scriptstyle \pm 0.09}$ & $83.97  {\scriptstyle \pm 0.44}$ & $83.16  {\scriptstyle \pm 0.29}$ & $68.85  {\scriptstyle \pm 1.28}$ & $71.59  {\scriptstyle \pm 1.01}$ \\ 
  & BA-Orth & $93.37  {\scriptstyle \pm 0.09}$ & $92.13  {\scriptstyle \pm 0.14}$ & $91.57  {\scriptstyle \pm 0.04}$ & $83.02  {\scriptstyle \pm 0.27}$ & $83.89  {\scriptstyle \pm 0.27}$ \\ 
  & WBA-Orth & $93.37  {\scriptstyle \pm 0.11}$ & $92.13  {\scriptstyle \pm 0.14}$ & $91.56  {\scriptstyle \pm 0.04}$ & $83.06  {\scriptstyle \pm 0.27}$ & $83.98  {\scriptstyle \pm 0.24}$ \\ 
  & MV-Orth & $93.30  {\scriptstyle \pm 0.09}$ & $92.16  {\scriptstyle \pm 0.10}$ & $91.55  {\scriptstyle \pm 0.09}$ & $82.61  {\scriptstyle \pm 0.42}$ & $83.73  {\scriptstyle \pm 0.21}$ \\ 
  & BA-OAC & $\mathbf{93.80  {\scriptstyle \pm 0.10}}$ & $92.32  {\scriptstyle \pm 0.09}$ & $91.79  {\scriptstyle \pm 0.10}$ & $83.29  {\scriptstyle \pm 0.23}$ & $83.95  {\scriptstyle \pm 0.19}$ \\ 
  & WBA-OAC & $\mathbf{93.80  {\scriptstyle \pm 0.11}}$ & $92.31  {\scriptstyle \pm 0.07}$ & $91.80  {\scriptstyle \pm 0.10}$ & $83.28  {\scriptstyle \pm 0.22}$ & $84.03  {\scriptstyle \pm 0.17}$ \\ 
  & MV-OAC & $93.62  {\scriptstyle \pm 0.06}$ & $92.28  {\scriptstyle \pm 0.11}$ & $91.66  {\scriptstyle \pm 0.06}$ & $82.92  {\scriptstyle \pm 0.26}$ & $83.74  {\scriptstyle \pm 0.17}$ \\ 
 \midrule  
 \multirow{7}{*}{$5$} & Best Client & $19.20  {\scriptstyle \pm 0.28}$ & $19.24  {\scriptstyle \pm 0.62}$ & $18.37  {\scriptstyle \pm 0.13}$ & $16.86  {\scriptstyle \pm 0.64}$ & $17.09  {\scriptstyle \pm 0.36}$ \\  
  & BA-Orth & $31.00  {\scriptstyle \pm 0.48}$ & $47.93  {\scriptstyle \pm 0.52}$ & $53.36  {\scriptstyle \pm 0.34}$ & $57.37  {\scriptstyle \pm 0.43}$ & $64.12  {\scriptstyle \pm 0.85}$ \\  
  & WBA-Orth & $30.99  {\scriptstyle \pm 0.49}$ & $47.92  {\scriptstyle \pm 0.50}$ & $53.36  {\scriptstyle \pm 0.33}$ & $57.37  {\scriptstyle \pm 0.45}$ & $64.16  {\scriptstyle \pm 0.85}$ \\  
  & MV-Orth & $35.24  {\scriptstyle \pm 0.54}$ & $57.77  {\scriptstyle \pm 0.58}$ & $65.08  {\scriptstyle \pm 0.46}$ & $67.45  {\scriptstyle \pm 0.59}$ & $79.13  {\scriptstyle \pm 0.50}$ \\  
  & BA-OAC & $78.91  {\scriptstyle \pm 0.47}$ & $91.74  {\scriptstyle \pm 0.16}$ & $91.40  {\scriptstyle \pm 0.15}$ & $83.07  {\scriptstyle \pm 0.21}$ & $83.92  {\scriptstyle \pm 0.20}$ \\  
  & WBA-OAC & $78.91  {\scriptstyle \pm 0.45}$ & $91.75  {\scriptstyle \pm 0.15}$ & $91.39  {\scriptstyle \pm 0.14}$ & $83.09  {\scriptstyle \pm 0.20}$ & $84.03  {\scriptstyle \pm 0.21}$ \\  
  & MV-OAC & $85.36  {\scriptstyle \pm 0.22}$ & $\mathbf{92.14  {\scriptstyle \pm 0.11}}$ & $91.60  {\scriptstyle \pm 0.07}$ & $82.82  {\scriptstyle \pm 0.26}$ & $83.74  {\scriptstyle \pm 0.15}$ \\  
 \midrule  
 \multirow{7}{*}{$1$} & Best Client & $11.97  {\scriptstyle \pm 0.16}$ & $11.93  {\scriptstyle \pm 0.43}$ & $11.55  {\scriptstyle \pm 0.12}$ & $11.32  {\scriptstyle \pm 0.23}$ & $11.27  {\scriptstyle \pm 0.29}$ \\  
  & BA-Orth & $13.78  {\scriptstyle \pm 0.30}$ & $16.18  {\scriptstyle \pm 0.19}$ & $17.21  {\scriptstyle \pm 0.26}$ & $18.25  {\scriptstyle \pm 0.39}$ & $20.86  {\scriptstyle \pm 0.37}$ \\  
  & WBA-Orth & $13.78  {\scriptstyle \pm 0.30}$ & $16.18  {\scriptstyle \pm 0.18}$ & $17.21  {\scriptstyle \pm 0.27}$ & $18.25  {\scriptstyle \pm 0.39}$ & $20.86  {\scriptstyle \pm 0.36}$ \\  
  & MV-Orth & $14.44  {\scriptstyle \pm 0.31}$ & $17.63  {\scriptstyle \pm 0.06}$ & $19.31  {\scriptstyle \pm 0.14}$ & $21.12  {\scriptstyle \pm 0.40}$ & $29.36  {\scriptstyle \pm 0.28}$ \\  
  & BA-OAC & $22.46  {\scriptstyle \pm 0.38}$ & $58.45  {\scriptstyle \pm 0.64}$ & $71.14  {\scriptstyle \pm 0.37}$ & $78.34  {\scriptstyle \pm 0.45}$ & $82.34  {\scriptstyle \pm 0.39}$ \\  
  & WBA-OAC & $22.47  {\scriptstyle \pm 0.37}$ & $58.47  {\scriptstyle \pm 0.65}$ & $71.14  {\scriptstyle \pm 0.37}$ & $78.34  {\scriptstyle \pm 0.40}$ & $82.41  {\scriptstyle \pm 0.39}$ \\  
  & MV-OAC & $25.20  {\scriptstyle \pm 0.27}$ & $69.74  {\scriptstyle \pm 0.51}$ & $82.43  {\scriptstyle \pm 0.33}$ & $81.05  {\scriptstyle \pm 0.45}$ & $\mathbf{83.52  {\scriptstyle \pm 0.22}}$ \\  
 \bottomrule
\end{tabular}
\label{tab:ablation_numusers}
\end{table*}
\Cref{tab:ablation_numusers} demonstrates the scalability of the proposed method on CIFAR-10 dataset for various number of clients. As the number of clients increases, the performance of all methods tends to decline, primarily due to the assumption that the dataset is split into non-overlapping portions among different clients. This leads to less diverse data at each user, reducing the effectiveness of individual models in capturing the overall data distribution. The {\it Best Client} method, which relies on the model of a single user, performs well initially but experiences a sharp decline in Macro-F1 scores, particularly as privacy constraints are introduced. For instance, at $\varepsilon = 5$, its performance drops from 19.20 with 5 clients to 17.09 with 100 clients, and similarly at $\varepsilon = 1$, from 11.97 to 11.27. This demonstrates the method's limited scalability, especially in scenarios where privacy is essential and the number of clients grows.

In contrast, \gls{OAC} Methods--{\it BA-OAC}, {\it WBA-OAC}, and {\it MV-OAC}--exhibit better scalability, especially under privacy constraints. These methods leverage the superposition property of multiple access channels, allowing them to efficiently aggregate information from users without needing orthogonal communication channels. Even with 100 users and at the strictest privacy level ($\varepsilon = 1$), {\it MV-OAC} still maintains a high Macro-F1 score of 83.52. This makes \gls{OAC} methods more resilient to increasing numbers of users and privacy requirements compared to orthogonal methods, which experience more significant performance drops under the same conditions. Overall, OAC-based methods provide a more scalable and privacy-efficient solution for distributed inference in large-scale, privacy-sensitive environments.

\subsection{Statistical Significance of the Results}

In this section, we demonstrate the statistical significance of the results using a non-parametric test.

We first perform Friedman test and obtain p-value less than $0.05$, which means that distributions are significantly different. We then proceed with the post-hoc Nemenyi test at $0.05$ significance level, in which we found the \gls{CD} as \num{1.42}, which is the minimum difference to have between two methods for statistical distinguishability. We can reject the null hypothesis for the models with an average rank larger than the \gls{CD}, i.e., it means these models are drawn from significantly different distributions.

The top plot in \cref{fig:nemenyi} shows the \gls{CD} diagram for the non-private case, $\varepsilon=\infty$. It shows that {\it \gls{BA-OAC}} and {\it \gls{WBA-OAC}} perform significantly better than all the other methods. The second and the third plots in \cref{fig:nemenyi} show the \gls{CD} diagram for the moderately private $(\varepsilon=5)$ and strongly private $(\varepsilon=1)$ cases, respectively. In both, {\it \gls{MV-OAC}} becomes the method with the highest rank due to the inclusion of privacy noise, and all the \gls{OAC}-based methods perform significantly better than other alternatives. There is also no significant difference among orthogonal methods and among \gls{OAC}-based methods. In all three scenarios, the {\it Best Client} method performs the poorest due to the performance improvement of ensembling.

\begin{figure}
\centering
\begin{tikzpicture}
\node (Label) at (1.4674368378252352, 0.7){\tiny{CD = 1.42}}; 
\node (Label) at (3.5, 0.7){$\varepsilon=\infty$};
\draw[decorate,decoration={snake,amplitude=.4mm,segment length=1.5mm,post length=0mm},very thick, color = black] (0.8571428571428571,0.5) -- (2.077730818507613,0.5);
\foreach \x in {0.8571428571428571, 2.077730818507613} \draw[thick,color = black] (\x, 0.4) -- (\x, 0.6);
 
\draw[gray, thick](0.8571428571428571,0) -- (6.0,0);
\foreach \x in {0.8571428571428571,1.7142857142857142,2.5714285714285716,3.4285714285714284,4.285714285714286,5.142857142857143,6.0} \draw (\x cm,1.5pt) -- (\x cm, -1.5pt);
\node (Label) at (0.8571428571428571,0.2){\tiny{1}};
\node (Label) at (1.7142857142857142,0.2){\tiny{2}};
\node (Label) at (2.5714285714285716,0.2){\tiny{3}};
\node (Label) at (3.4285714285714284,0.2){\tiny{4}};
\node (Label) at (4.285714285714286,0.2){\tiny{5}};
\node (Label) at (5.142857142857143,0.2){\tiny{6}};
\node (Label) at (6.0,0.2){\tiny{7}};
\draw[decorate,decoration={snake,amplitude=.4mm,segment length=1.5mm,post length=0mm},very thick, color = black](1.4178571428571427,-0.25) -- (1.5821428571428573,-0.25);
\draw[decorate,decoration={snake,amplitude=.4mm,segment length=1.5mm,post length=0mm},very thick, color = black](3.121428571428572,-0.4) -- (3.842857142857142,-0.4);
\draw[decorate,decoration={snake,amplitude=.4mm,segment length=1.5mm,post length=0mm},very thick, color = black](3.7428571428571424,-0.55) -- (4.764285714285714,-0.55);
\node (Point) at (1.4678571428571427, 0){};\node (Label) at (0.25,-0.8500000000000001){\scriptsize{BA-OAC}}; \draw (Point) |- (Label);
\node (Point) at (1.5321428571428573, 0){};\node (Label) at (0.25,-1.1500000000000001){\scriptsize{WBA-OAC}}; \draw (Point) |- (Label);
\node (Point) at (3.1714285714285717, 0){};\node (Label) at (0.25,-1.4500000000000002){\scriptsize{BA-Orth}}; \draw (Point) |- (Label);
\node (Point) at (6.0, 0){};\node (Label) at (6.75,-0.8500000000000001){\scriptsize{Best Client}}; \draw (Point) |- (Label);
\node (Point) at (4.714285714285714, 0){};\node (Label) at (6.75,-1.1500000000000001){\scriptsize{MV-Orth}}; \draw (Point) |- (Label);
\node (Point) at (3.7928571428571423, 0){};\node (Label) at (6.75,-1.4500000000000002){\scriptsize{MV-OAC}}; \draw (Point) |- (Label);
\node (Point) at (3.3214285714285716, 0){};\node (Label) at (6.75,-1.75){\scriptsize{WBA-Orth}}; \draw (Point) |- (Label);
\end{tikzpicture}
\begin{tikzpicture}
\node (Label) at (1.4674368378252352, 0.7){\tiny{CD = 1.42}}; 
\node (Label) at (3.5, 0.7){$\varepsilon=5$};
\draw[decorate,decoration={snake,amplitude=.4mm,segment length=1.5mm,post length=0mm},very thick, color = black] (0.8571428571428571,0.5) -- (2.077730818507613,0.5);
\foreach \x in {0.8571428571428571, 2.077730818507613} \draw[thick,color = black] (\x, 0.4) -- (\x, 0.6);
 
\draw[gray, thick](0.8571428571428571,0) -- (6.0,0);
\foreach \x in {0.8571428571428571,1.7142857142857142,2.5714285714285716,3.4285714285714284,4.285714285714286,5.142857142857143,6.0} \draw (\x cm,1.5pt) -- (\x cm, -1.5pt);
\node (Label) at (0.8571428571428571,0.2){\tiny{1}};
\node (Label) at (1.7142857142857142,0.2){\tiny{2}};
\node (Label) at (2.5714285714285716,0.2){\tiny{3}};
\node (Label) at (3.4285714285714284,0.2){\tiny{4}};
\node (Label) at (4.285714285714286,0.2){\tiny{5}};
\node (Label) at (5.142857142857143,0.2){\tiny{6}};
\node (Label) at (6.0,0.2){\tiny{7}};
\draw[decorate,decoration={snake,amplitude=.4mm,segment length=1.5mm,post length=0mm},very thick, color = black](1.1285714285714286,-0.25) -- (2.0964285714285715,-0.25);
\draw[decorate,decoration={snake,amplitude=.4mm,segment length=1.5mm,post length=0mm},very thick, color = black](3.6357142857142857,-0.4) -- (4.657142857142857,-0.4);
\node (Point) at (1.1785714285714286, 0){};\node (Label) at (0.25,-0.65){\scriptsize{MV-OAC}}; \draw (Point) |- (Label);
\node (Point) at (1.9178571428571427, 0){};\node (Label) at (0.25,-0.95){\scriptsize{BA-OAC}}; \draw (Point) |- (Label);
\node (Point) at (2.0464285714285717, 0){};\node (Label) at (0.25,-1.25){\scriptsize{WBA-OAC}}; \draw (Point) |- (Label);
\node (Point) at (6.0, 0){};\node (Label) at (6.75,-0.65){\scriptsize{Best Client}}; \draw (Point) |- (Label);
\node (Point) at (4.607142857142857, 0){};\node (Label) at (6.75,-0.95){\scriptsize{WBA-Orth}}; \draw (Point) |- (Label);
\node (Point) at (4.564285714285715, 0){};\node (Label) at (6.75,-1.25){\scriptsize{BA-Orth}}; \draw (Point) |- (Label);
\node (Point) at (3.6857142857142855, 0){};\node (Label) at (6.75,-1.5499999999999998){\scriptsize{MV-Orth}}; \draw (Point) |- (Label);
\end{tikzpicture}
\begin{tikzpicture}
\node (Label) at (1.4674368378252352, 0.7){\tiny{CD = 1.42}}; 
\node (Label) at (3.5, 0.7){$\varepsilon=1$};
\draw[decorate,decoration={snake,amplitude=.4mm,segment length=1.5mm,post length=0mm},very thick, color = black] (0.8571428571428571,0.5) -- (2.077730818507613,0.5);
\foreach \x in {0.8571428571428571, 2.077730818507613} \draw[thick,color = black] (\x, 0.4) -- (\x, 0.6);
 
\draw[gray, thick](0.8571428571428571,0) -- (6.0,0);
\foreach \x in {0.8571428571428571,1.7142857142857142,2.5714285714285716,3.4285714285714284,4.285714285714286,5.142857142857143,6.0} \draw (\x cm,1.5pt) -- (\x cm, -1.5pt);
\node (Label) at (0.8571428571428571,0.2){\tiny{1}};
\node (Label) at (1.7142857142857142,0.2){\tiny{2}};
\node (Label) at (2.5714285714285716,0.2){\tiny{3}};
\node (Label) at (3.4285714285714284,0.2){\tiny{4}};
\node (Label) at (4.285714285714286,0.2){\tiny{5}};
\node (Label) at (5.142857142857143,0.2){\tiny{6}};
\node (Label) at (6.0,0.2){\tiny{7}};
\draw[decorate,decoration={snake,amplitude=.4mm,segment length=1.5mm,post length=0mm},very thick, color = black](0.9357142857142856,-0.25) -- (2.1392857142857142,-0.25);
\draw[decorate,decoration={snake,amplitude=.4mm,segment length=1.5mm,post length=0mm},very thick, color = black](3.5500000000000007,-0.4) -- (4.7214285714285715,-0.4);
\node (Point) at (0.9857142857142857, 0){};\node (Label) at (0.25,-0.65){\scriptsize{MV-OAC}}; \draw (Point) |- (Label);
\node (Point) at (2.067857142857143, 0){};\node (Label) at (0.25,-0.95){\scriptsize{BA-OAC}}; \draw (Point) |- (Label);
\node (Point) at (2.0892857142857144, 0){};\node (Label) at (0.25,-1.25){\scriptsize{WBA-OAC}}; \draw (Point) |- (Label);
\node (Point) at (5.957142857142857, 0){};\node (Label) at (6.75,-0.65){\scriptsize{Best Client}}; \draw (Point) |- (Label);
\node (Point) at (4.671428571428572, 0){};\node (Label) at (6.75,-0.95){\scriptsize{WBA-Orth}}; \draw (Point) |- (Label);
\node (Point) at (4.628571428571429, 0){};\node (Label) at (6.75,-1.25){\scriptsize{BA-Orth}}; \draw (Point) |- (Label);
\node (Point) at (3.6000000000000005, 0){};\node (Label) at (6.75,-1.5499999999999998){\scriptsize{MV-Orth}}; \draw (Point) |- (Label);
\end{tikzpicture}
\caption{Comparison of the methods for different privacy levels $(\varepsilon=\infty,\varepsilon=5,\varepsilon=1)$, respectively, in terms of average rank with the post-hoc Nemenyi test using a critical difference diagram. Connected methods are not significantly different at the $0.05$ significance level.}
\label{fig:nemenyi}
\end{figure}
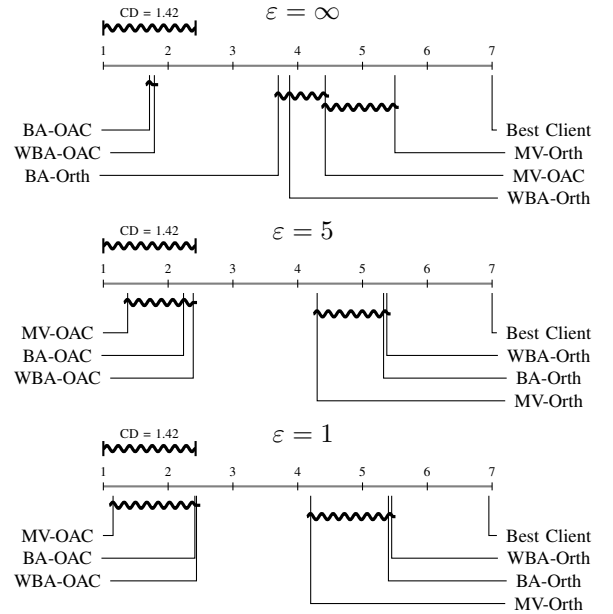

\subsection{Analog vs Digital Implementation}

In this section, we comment on the possibility of adapting the introduced methods in \cref{sec:comparison_baselines,sec:ablation_privacy_projection} to the digital setting.

In the non-private case, {\it Majority Voting} can be directly implemented digitally since we use a predefined code $\Pm$. Again, in the non-private case, {\it Belief Averaging}-based methods can be implemented by quantizing the beliefs, and {\it Weighted Belief Averaging}-based methods can be implemented after quantizing the weighted beliefs. These methods cannot be directly implemented in the cases with privacy noise, i.e., in the cases when $\varepsilon \in \LP 1, 5 \RP$. In the weak-private and private cases, {\it Orthogonal Projection of RR}, introduced in \cref{sec:ablation_privacy_projection}, can be implemented digitally since privacy noise is digital as well.

\begin{remark}
For simplicity, we employ analog \gls{OAC} in our system model and focus on the private collaborative inference task. To enable compatibility with digital communication systems, the digital \gls{OAC} frameworks proposed in \cite{razavikia2023channelcomp,csahin2023over,qiao2024massive} provide a strong foundational basis. We leave this adaptation for future exploration and encourage readers to consult the survey papers \cite{perez2024waveforms, csahin2023survey} for additional perspectives.
\end{remark}


\subsection{Practical Considerations}
\textbf{Synchronization and Mobility:} In practical deployments, achieving perfect time and frequency synchronization among a large number of devices is challenging due to factors like clock drift, network delays, hardware limitations, and device mobility~\cite{csahin2023survey}. As devices move during the data consensus process, changes in network topology and neighborhood relationships can occur, which may impact system performance by introducing additional asynchrony and communication delays. By incorporating participation probability into our method, we not only improve privacy guarantees but also model the impact of asynchrony and device mobility. This approach provides a practical assessment of the proposed method's resilience to dynamic network conditions, offering insights into its effectiveness in real-world scenarios where perfect synchronization and static network topologies are unattainable.

\textbf{Instantaneous \gls{CSI}}:
In this work, we assume each client has access to its instantaneous \gls{CSI} (channel gain), enabling transmission power scaling for optimal signal aggregation at the inference server. However, acquiring real-time \gls{CSI} can introduce delays and increase signaling overhead, especially in high-mobility or large-scale networks. To address this, recent research explores blind \gls{OAC} methods that bypass per-user instantaneous \gls{CSI} by leveraging statistical channel information or adaptive transmission strategies.

Blind \gls{OAC} designs~\cite{amiri2021blind,csahin2023survey} aim to reduce signaling overhead and latency by eliminating the need for real-time \gls{CSI} exchange. Instead, these methods rely on techniques such as power control (adjusting transmission power based on long-term channel statistics or pre-configured heuristics), fixed modulation schemes (using predetermined modulation formats independent of real-time channel conditions), or statistical estimations based on historical channel data. While our current approach leverages \gls{CSI} for accurate signal alignment, blind \gls{OAC} strategies offer a compelling solution for dense or high-mobility networks, where signaling delays can severely impact performance. Exploring how to adapt our method to incorporate blind \gls{OAC} principles is an exciting direction for future research.

\textbf{Energy Efficiency:} Our framework includes a channel gain threshold for optimizing transmission power, ensuring that clients in poor channel conditions reduce their energy expenditure. However, as the probability of non-participation due to poor channel gains is extremely low, the threshold has minimal impact on overall client participation. This design effectively balances energy efficiency with fairness, allowing nearly all clients to participate equitably in the inference process without significantly affecting power constraints. Additionally, in our system model, we assume that \gls{iid} $h_{i,t}$'s are \gls{iid}, in which case the participation and performance remains fair. Thus, our approach ensures sustainable operation across \gls{IoT} networks, maintaining high client participation while minimizing energy usage. For a deeper exploration on energy efficiency, we refer the reader to~\cite{wang2024over}.

\textbf{Common Randomness:} In our method, clients use a pseudorandom seed to determine participation based on a shared common randomness, eliminating the need for inter-client communication. This setup is particularly beneficial in preserving privacy as it reduces communication overhead and potential exposure of client status or model information. However, we recognize that under non-ideal conditions, achieving exact synchronization is challenging. To mitigate this, our model incorporates a probabilistic participation parameter $p$ that approximates asynchronous conditions by probabilistically excluding clients from inference rounds. This probabilistic exclusion serves as a proxy for potential synchronization errors, thus modeling the impact of communication delays or asynchrony that may prevent clients from transmitting in real-time. These approaches, combined with our probabilistic participation model, enhance our framework’s adaptability to dynamic network conditions, such as mobility and varying signal quality, ensuring robustness even when perfect synchronization cannot be achieved.
\section{Conclusion}
\label{sec:conclusion}
We have introduced a private collaborative inference framework at the wireless edge. We have exploited \gls{OAC} for bandwidth-efficient and accurate multi-class classification for multiple views while preserving the privacy of the local models and their data.  We have provided DP guarantees exploiting both distributed noise addition after projection and random participation. We have systematically evaluated the introduced methods with OAC and shown that distributed edge inference with OAC performs significantly better than its orthogonal counterpart while using fewer resources. We have observed that while transmitting local beliefs from each client is more informative, enforcing and transmitting hard local decisions can be more reliable when noise is introduced to guarantee privacy.

\balance
\bibliographystyle{IEEEtran}
\bibliography{main}

\begin{IEEEbiography}[{\includegraphics[width=1in,height=1.25in,clip,keepaspectratio]{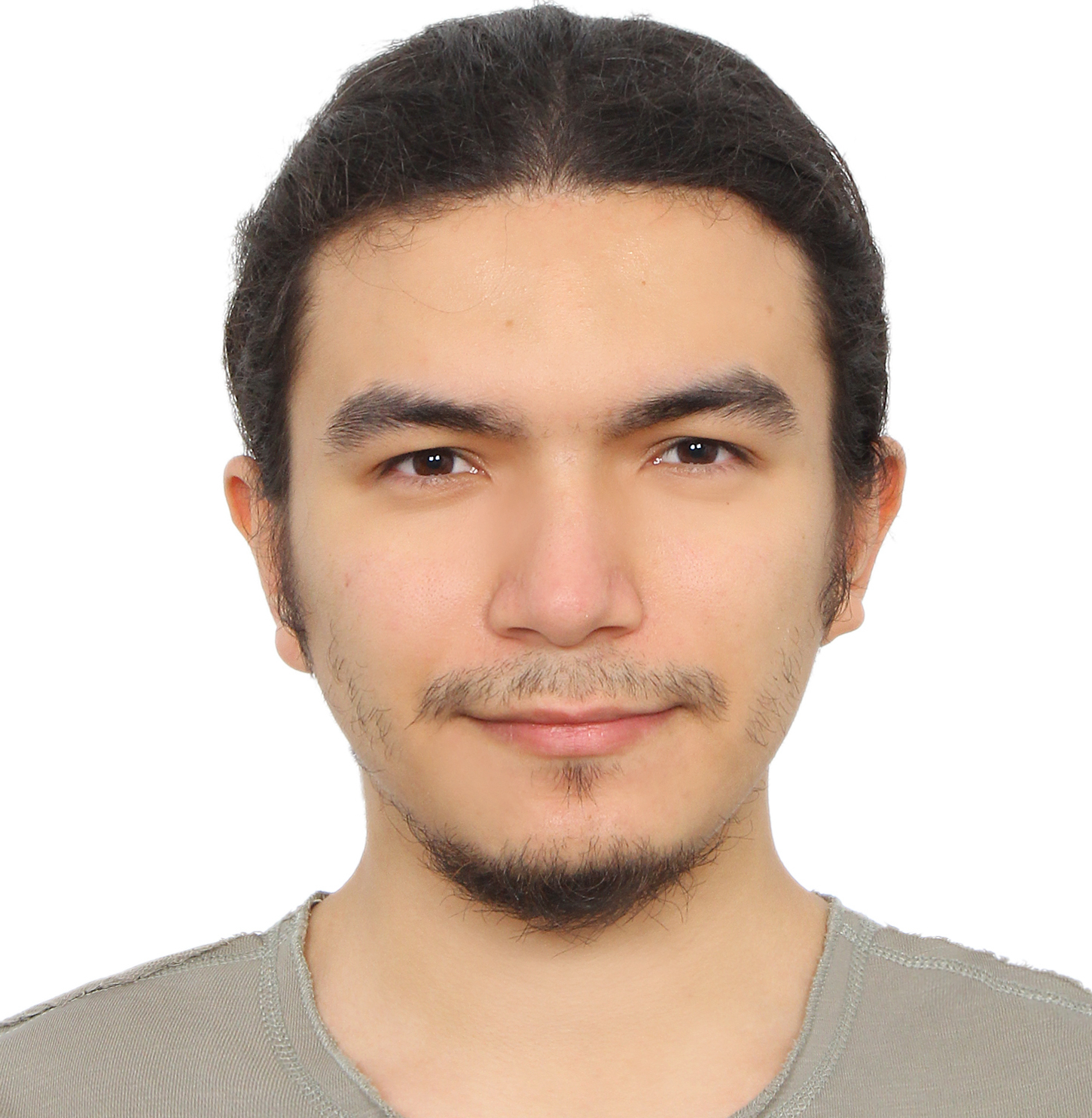}}]{Selim F. Yilmaz} received the B.S. in computer engineering and M.S. degree in electrical and electronic engineering from Bilkent University (Türkiye), in 2019 and 2021, respectively.
He is currently a PhD student at the Dept. of Electrical and Electronic Engineering of Imperial College London (United Kingdom). His main research interests include edge learning and inference, and source-channel coding.
\end{IEEEbiography}

\begin{IEEEbiography}[{\includegraphics[width=1in,height=1.25in,clip,keepaspectratio]{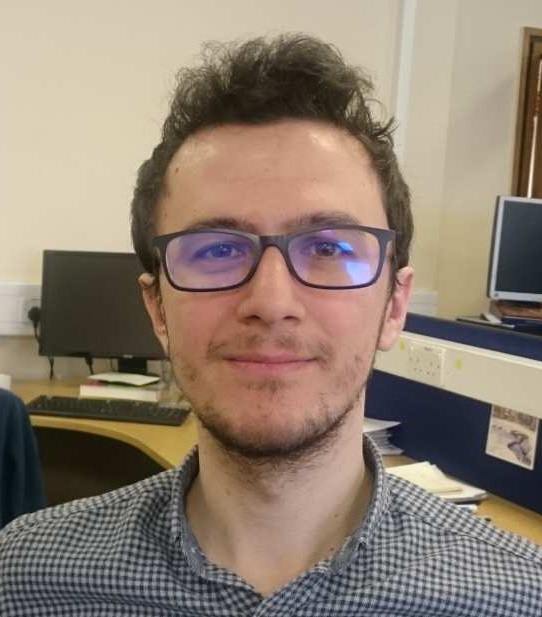}}]{Burak Hasırcıoğlu}
 received his B.Sc. degree in Electrical and Electronics Engineering from Middle East Technical University (METU), Ankara, Turkey in 2014, M.Sc. degree in Communication Systems from Ecole Polytechnique Fédérale de Lausanne (EPFL), Switzerland in 2017, and Ph.D. degree in Electrical and Electronic Engineering from Imperial College London, UK in 2024. He is currently a research associate at The Alan Turing Institute, in London, UK. His research interests include large language models, cybersecurity, privacy, federated learning and coded computing.
\end{IEEEbiography}

\begin{IEEEbiography}[{\includegraphics[width=1in,height=1.25in,clip,keepaspectratio]{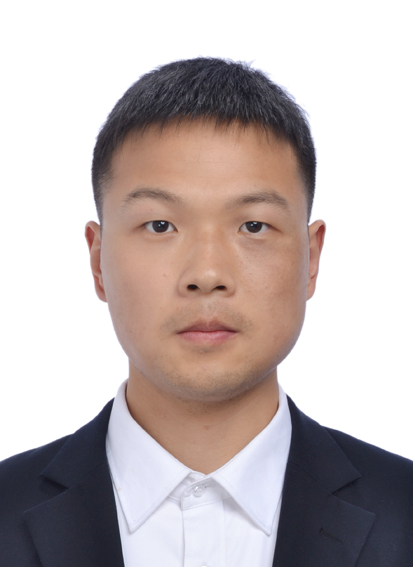}}]{Li Qiao}
(Graduate Student Member, IEEE) received the B.Sc. degree from the Beijing Institute of Technology in 2019, where he is currently pursuing the Ph.D. degree with the School of Information and Electronics. He is a collaborative Ph.D. Student with the 5GIC/6GIC Center, University of Surrey, and was a Visiting Student with the IPC Laboratory, Imperial College London. His current research interests include generative semantic communications, massive communications, and federated edge intelligence.
\end{IEEEbiography}

\begin{IEEEbiography}[{\includegraphics[width=1in,height=1.25in,clip,keepaspectratio]{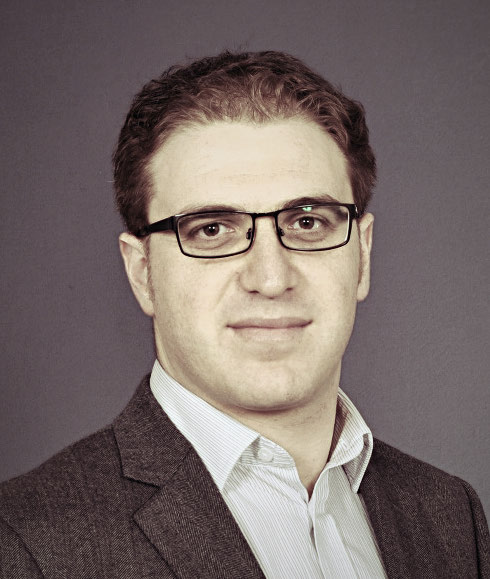}}]{Deniz Gündüz}
(Fellow, IEEE) received the B.S. degree in electrical and electronics engineering from METU, Ankara, Turkey, in 2002, and the M.S. and Ph.D. degrees in electrical engineering from the NYU Tandon School of Engineering, Brooklyn, NY, USA, in 2004 and 2007, respectively. He was a Research Assistant Professor with Stanford University, Palo Alto, CA, USA, from 2007 to 2009. He was a Postdoctoral Researcher from 2007 to 2009 and a Visiting Researcher from 2009 to 2011 with Princeton University, Princeton, NJ, USA. He was a Research Associate with the Centre Tecnologic de Telecomunicacions de Catalunya (CTTC), Barcelona, Spain, from 2009 to 2012. In 2012, he joined the Electrical and Electronic Engineering Department, Imperial College London, London, U.K., where he is currently a Professor of information processing and serves as the Deputy Head of the Intelligent Systems and Networks Group. He was a Visiting Professor with the University of Padova, Padua, Italy, in 2018 and 2020. He was a Part-Time Faculty Member with the University of Modena and Reggio Emilia, Modena, Italy, from 2019 to 2022. He has co-authored several award-winning articles, including the IEEE Communications Society—Young Author Best Paper Award in 2022 and the IEEE International Conference on Communications Best Paper Award in 2023. His research interests lie in the areas of communications, information theory, machine learning, and privacy. Dr. Gündüz is an Elected Member of the IEEE Signal Processing Society Signal Processing for Communications and Networking (SPCOM) and Machine Learning for Signal Processing (MLSP) Technical Committees, and has been serving as the Chair for the IEEE Information Theory Society U.K. and Ireland Chapter since 2022. He serves as an Area Editor for IEEE TRANSACTIONS ON INFORMATION THEORY and IEEE TRANSACTIONS ON COMMUNICATIONS. He was a recipient of the IEEE Communications Society—Communication Theory Technical Committee (CTTC) Early Achievement Award in 2017, and Starting (2016), Consolidator (2022), and Proof-of-Concept (2023) Grants of the European Research Council (ERC). He received the Imperial College London—President’s Award for Excellence in Research Supervision in 2023.
\end{IEEEbiography}

\end{document}